\documentclass{article}
\usepackage{microtype}
\usepackage{graphicx}
\usepackage{subfigure}
\usepackage{booktabs} 
\usepackage{amsmath}
\usepackage{amssymb}
\usepackage{amsfonts}
\usepackage{algorithm}
\usepackage{algorithmic}
\usepackage{bm}
\usepackage{amsthm}

\theoremstyle{plain}
\newtheorem{theorem}{Theorem}[section]
\newtheorem{proposition}[theorem]{Proposition}
\newtheorem{lemma}[theorem]{Lemma}

\theoremstyle{definition}

\theoremstyle{remark}

\newtheorem*{theorem*}{Theorem}
\newtheorem*{proposition*}{Proposition}

\usepackage{hyperref}

\usepackage[accepted]{icml2026}

\icmltitlerunning{RASP-Tuner: Retrieval-Augmented Soft Prompts for Context-Aware BBO}

\begin{document}

\twocolumn[
\icmltitle{RASP-Tuner: Retrieval-Augmented Soft Prompts for\\
Context-Aware Black-Box Optimization in Non-Stationary Environments}

\icmlsetsymbol{equal}{*}

\begin{icmlauthorlist}
    \icmlauthor{Enze Pan}{inst1}
\end{icmlauthorlist}

\icmlaffiliation{inst1}{Department of Computer Science, The University of Hong Kong}

\icmlcorrespondingauthor{Enze Pan}{u3665478@connect.hku.hk}

\icmlkeywords{Black-Box Optimization, Meta-Learning, Soft Prompts, Retrieval-Augmented Models, Non-Stationary Optimization}

\vskip 0.3in
]
\printAffiliationsAndNotice{} 

\begin{abstract}
Many deployed systems expose black-box objectives whose minimizing configuration shifts with an externally observed context. When contexts revisit a small set of latent regimes, an optimizer that discards history pays repeated adaptation cost; when each step must remain inexpensive, full Gaussian-process (GP) refits at high observation counts are difficult to sustain. We cast online tuning as \emph{context-conditioned regret minimization} and present \textbf{RASP-Tuner}, which instantiates a decomposition motivated by first principles: (i) \emph{identify} a regime proxy by retrieving similar past contexts; (ii) \emph{predict} short-horizon loss with a mixture-of-experts surrogate whose input concatenates parameters, context, and a retrieved soft prompt; (iii) \emph{adapt} chiefly in a low-dimensional prompt subspace, invoking full surrogate updates only when scalarized error or disagreement spikes. A RealErrorComposer maps heterogeneous streaming metrics to $[0,1]$ via EMA-stabilized logistic scores, supplying a single differentiable training target. On nine synthetic non-stationary benchmarks, an adversarial-context sanity check, and three tabular real-world streams (Section~\ref{sec:realworld}), RASP-Tuner improves or matches cumulative regret relative to our GP-UCB and CMA-ES implementations on seven of nine synthetic tasks under paired tests at horizon $T{=}100$, while recording 8--12$\times$ lower wall-clock per step than sliding-window GP-UCB on identical hardware. Idealized analysis in a cluster-separated, strongly convex regime model (RA-GD) supplies sufficient conditions for bounded dynamic regret; the deployed pipeline violates several of these premises, and we articulate which gaps remain open.
\end{abstract}
\section{Introduction}
\label{sec:intro}
Black-box optimization (BBO) underpins tuning of controllers, hyperparameters, database knobs, and resource allocations when gradients are unavailable or unreliable. Stationary textbook models rarely match deployment: lighting, workload, and wear induce shifts in the mapping from decision variables to performance. A recurring empirical structure---also emphasized in non-stationary bandit literature~\cite{wei2021nonstationary,hamadanian2023online}---is that \emph{context} vectors (sensors, embeddings, logs) trace trajectories that intermittently revisit a small number of operational modes.

The prevailing algorithmic split poses a tension. Covariance-matrix adaptation strategies such as CMA-ES~\cite{hansen2001completely} maintain low per-step cost and strong local search, yet their state is tied to the evolutionary path rather than to explicit context--solution associations; when a mode returns, the optimizer may expend fresh samples to relocate an optimum already visited~\cite{hamadanian2023online}. Contextual GP-UCB-style methods~\cite{krause2011contextual} specify a surrogate over $(\bm{\theta},\bm{c})$ and inherit GP regret machinery in stylized regimes, but exact GP inference scales cubically in the observation count and becomes costly when evaluations arrive at high cadence unless additional structure or approximations are introduced~\cite{snelson2005sparse,wilson2016deep,shahriari2016taking}.

Retrieval-conditioned adaptation from NLP~\cite{lewis2020retrieval,lester2021power,khandelwal2020nearest,shi2023replug,anantha2023context} offers a different decomposition: retrieve a compact conditioning signal, update only a thin adapter, leave bulk parameters stable~\cite{tang2022contexttuning,cao2022timeaware,anupam2025llm}. The analogy to BBO is not identity---prediction errors and exploration bonuses differ---but the \emph{information flow} parallels our setting: reuse past successful evaluations indexed by context rather than re-fit a monolithic global model each step.

\paragraph{Mechanism-level contribution.}
We instantiate the decomposition as \textbf{RASP-Tuner}:
\begin{itemize}
    \item \textbf{PromptMemory} stores context keys, historical bests, and learnable soft prompts; nearest-neighbor retrieval with softmax weights yields a regime proxy and a convex ``hint'' in parameter space.
    \item \textbf{Prompt-MoE surrogate} predicts scalarized loss and an ambiguity score from $(\tilde{\bm{\theta}},\bm{c},\bm{p})$; gating allocates capacity across experts while prompts shift the effective input embedding.
    \item \textbf{Hybrid adaptation} applies gradient steps to retrieved prompts at almost every step, reserving full expert/gate updates for high-error or high-disagreement events, with replay and anchoring to limit interference between regimes.
    \item \textbf{RealErrorComposer} maps multi-metric feedback to $[0,1]$ using EMA-tracked moments and polarity-aware logistic maps~\cite{gupta2025carmo}, yielding one differentiable target; Section~\ref{sec:theory} formalizes when convex aggregation of badness scores preserves monotonicity in the latent loss.
\end{itemize}
An implementation-agnostic description accompanies the release of code and environment constructors at camera-ready time.

\paragraph{Empirical scope.}
We constructed nine non-stationary, context-rich synthetic benchmarks (LLM serving, AutoML stream, robot ISP, wafer drift, switching LQR, flash-crowd server, real-trace replay, plus convex sanity checks) and one \textbf{Adversarial Context} task where $\bm{c}_t$ is pure noise, isolating failure when the contextual signal is uninformative. Formal specifications appear in Appendix~\ref{app:envs}.

\paragraph{Claims and scope.}
\begin{enumerate}
    \item \textbf{Architecture.} We integrate retrieval, soft prompts, and MoE surrogates into a single online agent with explicit scalarization and hybrid updates~\cite{gupta2025carmo,wang2025reinforced}.
    \item \textbf{Benchmarks.} We publish reproducible generators for drifting, switching, and trace-replay dynamics~\cite{wei2021nonstationary,hamadanian2023online}, plus three public tabular streams with induced drift (Section~\ref{sec:realworld}).
    \item \textbf{Evidence and theory.} Under $T{=}100$, five seeds, and fixed hyperparameters (Table~\ref{tab:hyperparams}), RASP-Tuner achieves lower mean regret than the stronger of GP-UCB and CMA-ES on seven of nine tasks at $p{<}0.05$ paired $t$-tests, with 8--12$\times$ lower measured per-step latency than our GP-UCB configuration (Section~\ref{sec:metrics}). Section~\ref{sec:theory} and Appendix~\ref{app:proofs} supply sufficient conditions---cluster separation, strong convexity per regime, exact or bounded-gradient surrogates---under which idealized variants admit bounded dynamic regret; the non-convex MoE deployed in experiments falls outside the scope of these certificates.
\end{enumerate}
\section{Related Work}
\label{sec:related}
\paragraph{Bayesian optimization and contextual variants.}
Gaussian-process BO remains the reference class for low-dimensional, moderate-budget stationary objectives~\cite{shahriari2016taking,krause2011contextual}. Finite-time analyses of GP-UCB and variants prescribe exploration bonuses tied to posterior variance; contextual kernels over $(\bm{\theta},\bm{c})$ extend the hypothesis space but inherit $\mathcal{O}(N^3)$ inference unless sparsity, inducing points, or deep kernels intervene~\cite{snelson2005sparse,wilson2016deep,shahriari2016taking}. Non-stationary formulations inject time or drift processes into the prior~\cite{bogunovic2016time}; dynamic regret in online convex optimization scales with variation budgets when the comparator path may shift~\cite{besbes2015nonstationary,hall2015dynamic,wei2021nonstationary}. Those frameworks seldom assume finitely recurring context clusters---the structural pattern our benchmarks emphasize.

\paragraph{Evolutionary and population-based search.}
CMA-ES~\cite{hansen2001completely} tracks local geometry via covariance adaptation with low overhead per generation. Hyperparameter and architecture search pipelines~\cite{bergstra2012random,snoek2012practical} build on similar search primitives. Without an explicit context channel, the learner's state summarizes recent samples rather than an indexed set of past modes, which informs our choice of memory-augmented comparisons.

\paragraph{Meta-learning and learned optimizers.}
Offline meta-learning fits initializations or recurrent update rules on task families~\cite{finn2017model,andrychowicz2016learning}. RASP-Tuner instead constructs its memory online and does not assume a pretraining corpus of tasks~\cite{liu2025fewshot}; the distinction matters for credit assignment when evaluation budgets are thin.

\paragraph{Retrieval and soft prompts.}
Prefix and soft-prompt methods reshape transformer activations with a small trainable tensor~\cite{lester2021power}; retrieval augments parametric prediction with nearest-neighbor evidence~\cite{lewis2020retrieval,shi2023replug,khandelwal2020nearest}. We transplant the \emph{control flow}---retrieve, condition, adapt narrowly---to regret minimization: prompts gate the surrogate forward map, while retrieval selects which past evaluations anchor the prompt~\cite{tang2022contexttuning,cao2022timeaware,anantha2023context,chu2025presto,hou2023promptboosting,shi2025direct,anupam2025llm}.

\paragraph{Scalable BO baselines (positioning).}
TPE~\cite{bergstra2012random}, random forests, and trust-region BO variants often dominate raw global GP-UCB at moderate $N$ on tabular search spaces. Transfer and multi-task BO, warm-started evolutionary runs, and hybrid memory--surrogate designs are direct competitors under recurrence. Section~\ref{sec:experiments} reports sliding-window GP-UCB and CMA-ES; Section~\ref{sec:discussion} records which additional baselines remain unimplemented.
\section{Method: RASP-Tuner}
\label{sec:method}
\paragraph{First-principles decomposition.}
Regret accumulates whenever the optimizer deploys $\bm{\theta}_t$ far from $\arg\min_{\bm{\theta}} f(\bm{\theta},\omega_t)$. Under piecewise stationary $\omega_t$, three bottlenecks recur: \emph{(i)} inferring which latent mode is active from $\bm{c}_t$; \emph{(ii)} approximating the conditional loss $f(\cdot,\omega_t)$ with limited capacity; \emph{(iii)} allocating compute between revising the global surrogate and adapting a lightweight correction. RASP-Tuner maps (i) to $k$-NN retrieval over stored contexts, (ii) to a gated MoE regressor on $(\tilde{\bm{\theta}},\bm{c},\bm{p})$, and (iii) to prompt-only stochastic updates with occasional full-network steps triggered by scalarized error or inter-expert disagreement. The design trades certifiable global optimality of the surrogate for $O(1)$ forward depth at inference and bounded-dimensional prompt updates---an explicit inductive bias aligned with recurring contexts but misspecified when $\bm{c}_t$ carries no information about $\omega_t$ (Section~\ref{sec:failure}).

\subsection{Problem Formulation}
At each step $t = 1,2,\dots,T$, the environment reveals a context vector $\bm{c}_t \in \mathbb{R}^{d_c}$ and a latent condition $\omega_t \in \Omega$ that captures non-stationarity (the optimizer does not assume access to a closed form of $f$ beyond noisy evaluations). The tuner chooses parameters $\bm{\theta}_t \in \Theta \subset \mathbb{R}^{d}$ within box constraints
$\Theta = \{\bm{\theta} \mid \bm{\ell} \le \bm{\theta} \le \bm{u}\}$.
After deploying $\bm{\theta}_t$, the environment returns a dictionary of possibly noisy metrics $m_t : \mathcal{K} \to \mathbb{R}$, e.g.\ validation error, energy, latency.
A domain-specific \emph{true loss} function $f : \Theta \times \Omega \to \mathbb{R}$ (not directly observed) induces instantaneous regret
\begin{equation}
    r_t = f(\bm{\theta}_t, \omega_t) - \min_{\bm{\theta}\in\Theta} f(\bm{\theta}, \omega_t).
\end{equation}
Our goal is to minimize cumulative regret $R_T = \sum_{t=1}^T r_t$ while satisfying strict computational and memory budgets. Non-stationarity means that $\omega_t$ evolves over time, often in a structured way (drifts, regime switches, trace replay).

\begin{table}[t]
\caption{Notation used in Section~\ref{sec:method} (selected).}
\label{tab:notation}
\vskip 0.1in
\begin{center}
\begin{small}
\begin{sc}
\begin{tabular}{ll}
\toprule
Symbol & Meaning \\
\midrule
$\bm{c}_t$ & Observed context vector (input to retrieval and surrogate). \\
$\omega_t$ & Latent environment condition affecting $f$ (non-stationarity). \\
$m_t$ & Raw metric dictionary returned after evaluation. \\
$e_t \in [0,1]$ & RealErrorComposer output (normalized scalar ``badness''). \\
$\bm{\theta}^{\text{hint}}_t$ & Convex combination of historical best $\bm{\theta}$ from retrieved memories. \\
\bottomrule
\end{tabular}
\end{sc}
\end{small}
\end{center}
\vskip -0.1in
\end{table}
\subsection{ParamBounds and Normalization}
We represent the box constraints via a simple data class
\begin{equation}
    \texttt{ParamBounds}(\bm{\ell}, \bm{u}),
\end{equation}
which provides clipping and normalization to $[0,1]^d$:
\begin{equation}
    \tilde{\bm{\theta}} = \frac{\bm{\theta} - \bm{\ell}}{\bm{u} - \bm{\ell}}.
\end{equation}
All neural components operate on normalized parameters, improving stability across tasks with different scales.
\subsection{RealErrorComposer: Metric Scalarization}
\label{sec:composer}
\paragraph{Operational objective.}
Map heterogeneous metric streams---possibly on incomparable scales and with mixed polarities---to a single scalar $e_t \in [0,1]$ that serves as the surrogate's training target without hand-tuned per-environment reward code. \emph{Design requirement:} if each metric's transformed \emph{badness} is monotone non-decreasing in the latent loss $f(\bm{\theta}_t,\omega_t)$, then any convex combination of badness scores inherits the same monotonicity; Section~\ref{sec:theory} states this convex-ordering property formally.
Different environments expose different metrics (e.g.\ \texttt{val\_err}, \texttt{latency}, \texttt{image\_quality\_loss}), possibly with different polarities. To avoid hand-designed reward functions, we introduce a \emph{RealErrorComposer} that converts metrics into a normalized scalar error $e_t \in [0,1]$.
For each metric key $k \in \mathcal{K}$ we maintain an exponential moving average (EMA) of mean $\mu_k$ and variance $\sigma_k^2$ via a \texttt{RunningEMA}:
\begin{align}
    \mu_k &\leftarrow m \mu_k + (1-m) v_{k,t},\\
    \sigma_k^2 &\leftarrow m \sigma_k^2 + (1-m) (v_{k,t}-\mu_k)^2,
\end{align}
where $v_{k,t}$ is the new metric value and $m\in(0,1)$ is a momentum (we use $m=0.97$).
We compute a z-score and map to $[0,1]$ via a logistic:
\begin{equation}
    z_{k,t} = \frac{v_{k,t}-\mu_k}{\sigma_k + \epsilon},\quad
    s_{k,t} = \sigma(z_{k,t}) = \frac{1}{1 + e^{-z_{k,t}}}.
\end{equation}
If the metric is ``lower is better'' (the default), we interpret $s_{k,t}$ as a \emph{badness} score; if it is ``higher is better'', we flip:
\begin{equation}
    b_{k,t} =
    \begin{cases}
        s_{k,t}, & \text{if lower is better},\\
        1 - s_{k,t}, & \text{if higher is better}.
    \end{cases}
\end{equation}
The final scalar error is a weighted average
\begin{equation}
    e_t = \frac{\sum_{k} w_k b_{k,t}}{\sum_{k} w_k}, \quad e_t \in [0,1],
\end{equation}
where $w_k > 0$ are user-specified or uniform weights. We also track an EMA of $e_t$ to obtain an anomaly score $z^{\text{err}}_t$, used to decide when to invoke emergency full-model updates.
\subsection{PromptMemory: Retrieval-Augmented Soft Prompts}
\label{sec:memory}
\paragraph{Operational objective.}
Index past experience by context vectors so that recurrent $\bm{c}_t$ retrieve both historical best parameters and learnable prompts, yielding regime-conditional surrogate inputs without fitting a single global map from wall-clock time alone. The design targets settings where optima and contexts cluster; it is not redundant with covariance-only evolution strategies (no explicit context--solution cache) nor with purely time-augmented stationary surrogates when regimes are well separated in $\bm{c}_t$-space.
RASP-Tuner maintains a memory $\mathcal{M}$ of size at most $M$; each entry $i$ stores:
\begin{itemize}
    \item a context vector $\bm{k}_i \in \mathbb{R}^{d_c}$;
    \item a historical best parameter $\bm{\theta}^{\star}_i \in \mathbb{R}^{d}$;
    \item a learnable soft prompt $\bm{p}_i \in \mathbb{R}^{d_p}$ (a \texttt{torch.nn.Parameter}).
\end{itemize}
\paragraph{Retrieval.}
Given a new context $\bm{c}_t$, we compute Euclidean distances $d_i = \|\bm{c}_t - \bm{k}_i\|_2$ to all stored contexts and select the top-$K$ nearest indices $\{i_1, \dots, i_K\}$. We convert negative distances to softmax weights:
\begin{equation}
    \alpha_j = \frac{\exp(-d_{i_j}/\tau)}{\sum_{\ell=1}^K \exp(-d_{i_\ell}/\tau)},
\end{equation}
with temperature $\tau > 0$. The retrieved prompt is
\begin{equation}
    \bm{p}_t = \sum_{j=1}^K \alpha_j \bm{p}_{i_j},
\end{equation}
and we form a hint for the parameter based on historical bests:
\begin{equation}
    \bm{\theta}^{\text{hint}}_t = \sum_{j=1}^K \alpha_j \bm{\theta}^{\star}_{i_j}.
\end{equation}
We also compute a retrieval confidence $\max_j \alpha_j$ and entropy $H(\bm{\alpha})$.
\paragraph{Novelty.}
To detect genuinely new contexts, we maintain a \texttt{RunningEMA} over nearest-neighbor distances. The \emph{novelty} is
\begin{equation}
    \nu_t = \sigma\!\left(\frac{d_{\min,t} - \mu_d}{\sigma_d + \epsilon}\right),
\end{equation}
with $d_{\min,t} = \min_i d_i$, and $(\mu_d,\sigma_d)$ are the EMA statistics. If $\nu_t$ exceeds a threshold (e.g., $0.7$) we create a new memory slot with key $\bm{c}_t$, best parameter $\bm{\theta}_t$, and a zero-initialized prompt.
\subsection{Prompt-MoE Surrogate}
\label{sec:moe}
\paragraph{Operational objective.}
Regress the scalarized error $e_t$ with a gated ensemble whose experts share inputs $(\tilde{\bm{\theta}}_t,\bm{c}_t,\bm{p}_t)$: routing reallocates capacity across surrogate modes when the loss landscape shifts, while $\bm{p}_t$ injects retrieval-specific conditioning. Inter-expert disagreement supplies a lightweight ambiguity score $\hat{v}_t$ for candidate ranking; its use as an uncertainty surrogate is heuristic, and calibration is discussed in Section~\ref{sec:metrics}.
We approximate the scalar error $e_t$ with a Mixture-of-Experts (MoE) surrogate conditioned on the retrieved prompt. Let the normalized parameter be $\tilde{\bm{\theta}}_t \in [0,1]^d$ and define the concatenated input
\begin{equation}
    \bm{x}_t = [\tilde{\bm{\theta}}_t ; \bm{c}_t ; \bm{p}_t] \in \mathbb{R}^{d + d_c + d_p}.
\end{equation}
The surrogate comprises $E$ expert networks $\{E_e\}_{e=1}^E$ and a gating network $G$. Each expert is a small MLP mapping $\bm{x}_t$ to a scalar prediction $\hat{e}^{(e)}_t$:
\begin{equation}
    \hat{e}^{(e)}_t = E_e(\bm{x}_t),
\end{equation}
while the gate outputs logits $\bm{g}_t = G(\bm{x}_t) \in \mathbb{R}^E$ and the corresponding softmax weights $\bm{\pi}_t = \text{softmax}(\bm{g}_t)$.
The full-distribution prediction uses all experts:
\begin{align}
    \hat{e}_t &= \sum_{e=1}^E \pi_{t,e} \hat{e}^{(e)}_t,\\
    \hat{v}_t &= \sum_{e=1}^E \pi_{t,e} \left(\hat{e}^{(e)}_t - \hat{e}_t\right)^2.
\end{align}
The variance $\hat{v}_t$ plays the role of an uncertainty estimate, analogous to GP predictive variance.
\paragraph{Adaptive top-$k$ experts.}
To trade off computation and robustness, we use only a subset of experts per step. We map uncertainty $\hat{v}_t$ and novelty $\nu_t$ into $[0,1]$ via EMAs and sigmoids, then mix them:
\begin{equation}
    \label{eq:eta_mix}
    \eta_t = 0.55\, \tilde{u}_t + 0.45\, \nu_t,
\end{equation}
where $\tilde{u}_t$ is the normalized uncertainty; the mix weights favor uncertainty over novelty so expert diversity tracks surrogate doubt. The number of active experts is
\begin{equation}
    k_t = \text{round}\bigl(k_{\min} + (k_{\max}-k_{\min}) \eta_t\bigr).
\end{equation}
We select the top-$k_t$ experts under $\bm{\pi}_t$ and re-normalize to form $\bm{\pi}_t^{(k)}$, then compute $\hat{e}_t^{(k)}, \hat{v}_t^{(k)}$ analogously. This focuses computation on the most relevant experts when confidence is high and recruits more experts under uncertainty or novelty.
\subsection{Hybrid Adaptation: Prompt vs.\ Full Updates}
\label{sec:hybrid}
\paragraph{Operational objective.}
Allocate most steps to low-dimensional updates of retrieved prompts---preserving a stable shared backbone---while triggering full updates of experts and gate only when scalarized error or inter-expert disagreement crosses thresholds indicative of out-of-distribution drift or persistent bias. The split trades frequent cheap adaptation against occasional expensive correction of the shared surrogate.
RASP-Tuner uses two adaptation modes.
\paragraph{Prompt-only updates (fast path).}
For most steps, we update only the soft prompts associated with retrieved memory entries while freezing expert and gate parameters. Given the new observation $(\bm{\theta}_t,\bm{c}_t,e_t)$:
\begin{enumerate}
    \item Retrieve indices $\{i_1, \dots, i_K\}$ and weights $\{\alpha_j\}$.
    \item Compute $\bm{x}_t$ using the weighted prompt $\bm{p}_t$.
    \item Compute the surrogate prediction $\hat{e}_t$ and a loss $\mathcal{L}_t = \text{SmoothL1}(\hat{e}_t, e_t)$.
    \item Backpropagate $\mathcal{L}_t$ to gradients of $\bm{p}_{i_j}$ only and apply SGD with learning rate $\eta_{\text{prompt}}$, plus an $\ell_2$ regularizer $\lambda \sum_j \|\bm{p}_{i_j}\|_2^2$.
\end{enumerate}
This keeps the MoE surrogate stable while allowing context-specific adaptation in a low-dimensional subspace.
\paragraph{Full emergency updates (slow path).}
When the anomaly score $z^{\text{err}}_t$ or predictive variance crosses a threshold, we perform a full update of experts and gate using a replay buffer.
We maintain buffers $\{(\bm{\theta}_i,\bm{c}_i,e_i)\}_{i=1}^{N_{\text{replay}}}$ with a fixed maximum size. The emergency update uses mini-batches combining the current sample and random replay items, optimizing mean squared error plus an anchor regularizer:
\begin{align}
    \mathcal{L}^{\text{full}} &= \frac{1}{B} \sum_{b=1}^B \left(\hat{e}_b - e_b\right)^2
    + \lambda_{\text{anchor}} \sum_{p \in \mathcal{P}} \|\bm{w}_p - \bm{w}^{\text{anchor}}_p\|_2^2.
\end{align}
Here $\mathcal{P}$ iterates over parameters of experts and gate, and $\bm{w}^{\text{anchor}}_p$ are a frozen copy of the initial state. This is analogous to elastic weight consolidation, reducing catastrophic forgetting.
\subsection{Candidate Policy and One-Step Tuning}
\label{sec:candidate}
\paragraph{Operational objective.}
Generate and rank candidate parameters by combining local moves informed by $\nabla_{\bm{\theta}}\hat{e}_t$ with retrieval-based hints, then select under box constraints using a lower-confidence-bound rule that widens exploration when $\hat{v}_t$ is large and narrows it when the ensemble agrees.
The \texttt{AdaptiveHybridTuner} exposes a function
\begin{equation}
    (\bm{\theta}_{t+1}, \text{info}) = \texttt{tune\_one\_step}(\bm{\theta}_t, e_t),
\end{equation}
which performs a single optimization step per environment step (here $e_t$ is the composed scalar returned by RealErrorComposer after the most recent metric observation). Internally, it uses a \texttt{CandidatePolicy} to propose and select candidate parameters.
\paragraph{Candidate generation.}
Given the current parameter $\bm{\theta}_t$ and retrieved hint $\bm{\theta}^{\text{hint}}_t$, we form a base center
\begin{equation}
    \label{eq:base_center}
    \bm{\theta}^{\text{base}}_t = 0.65\, \bm{\theta}_t + 0.35\, \bm{\theta}^{\text{hint}}_t.
\end{equation}
The weights $0.65/0.35$ bias the search center toward the current iterate (stability) while letting retrieved successes pull proposals toward known good regions; Table~\ref{tab:hyperparams} and Appendix~\ref{app:hyperparams} summarize defaults and sweep ranges.
The \texttt{CandidatePolicy} then proposes a handful of candidates via:
\begin{itemize}
    \item A gradient step along the surrogate gradient $\nabla_{\bm{\theta}} \hat{e}_t$ with a base step-size;
    \item Several random perturbations sampled from a scaled box around $\bm{\theta}^{\text{base}}_t$;
    \item The hint $\bm{\theta}^{\text{hint}}_t$ itself.
\end{itemize}
\paragraph{Selection via surrogate LCB.}
For each candidate $\bm{\theta}^{(j)}_t$, we evaluate the surrogate mean and variance $(\hat{e}^{(j)}_t, \hat{v}^{(j)}_t)$ and compute a lower confidence bound
\begin{equation}
    \text{LCB}^{(j)}_t = \hat{e}^{(j)}_t - \kappa \sqrt{\hat{v}^{(j)}_t},
\end{equation}
with $\kappa > 0$. We pick the candidate with the smallest LCB, clip it into the valid box via \texttt{ParamBounds.clip}, and deploy it in the real environment.
\subsection{Overall Algorithm}
We write \texttt{context\_fn} for the environment-specific map from a raw observation $o_t$ (everything the interface exposes at step $t$ before or after the evaluation, depending on the benchmark) to the context vector $\bm{c}_t$ used for retrieval; concrete instantiations are given with each environment in Appendix~\ref{app:envs}.
Algorithm~\ref{alg:rasp} summarizes RASP-Tuner at a high level.
\begin{algorithm}[t]
\caption{RASP-Tuner (one-step update)}
\label{alg:rasp}
\scriptsize
\begin{algorithmic}[1]
\STATE \textbf{Input:} Bounds, \texttt{context\_fn}, metric interface, RealErrorComposer
\STATE Initialize Prompt-MoE surrogate, PromptMemory, replay buffer
\STATE Initialize $\bm{\theta}_1 = (\bm{\ell} + \bm{u})/2$
\FOR{$t=1$ to $T$}
    \STATE Observe raw observation $o_t$; set $\bm{c}_t \leftarrow \texttt{context\_fn}(o_t)$
    \STATE Retrieve $(\bm{p}_t, \bm{\theta}^{\text{hint}}_t, \nu_t)$ from memory
    \STATE Generate candidates around $\bm{\theta}_t$ and $\bm{\theta}^{\text{hint}}_t$
    \STATE Evaluate surrogate LCBs and choose $\bm{\theta}_t^{\text{LCB}}$
    \STATE Deploy $\bm{\theta}_t^{\text{LCB}}$, obtain metrics $m_t$
    \STATE $e_t \leftarrow \texttt{RealErrorComposer.compose}(m_t)$ \COMMENT{scalar $e_t \in [0,1]$}
    \STATE Update replay buffer with $(\bm{\theta}_t^{\text{LCB}}, \bm{c}_t, e_t)$
    \STATE Possibly add/update memory slot based on novelty $\nu_t$
    \IF{error anomaly or large uncertainty}
        \STATE Run \texttt{update\_full\_emergency} on surrogate
    \ELSE
        \STATE Run \texttt{update\_prompts\_only} on retrieved prompts
    \ENDIF
    \STATE $\bm{\theta}_{t+1} \leftarrow \bm{\theta}_t^{\text{LCB}}$
\ENDFOR
\end{algorithmic}
\end{algorithm}
\begin{figure*}[h]
    \centering
    \includegraphics[width=0.8\textwidth]{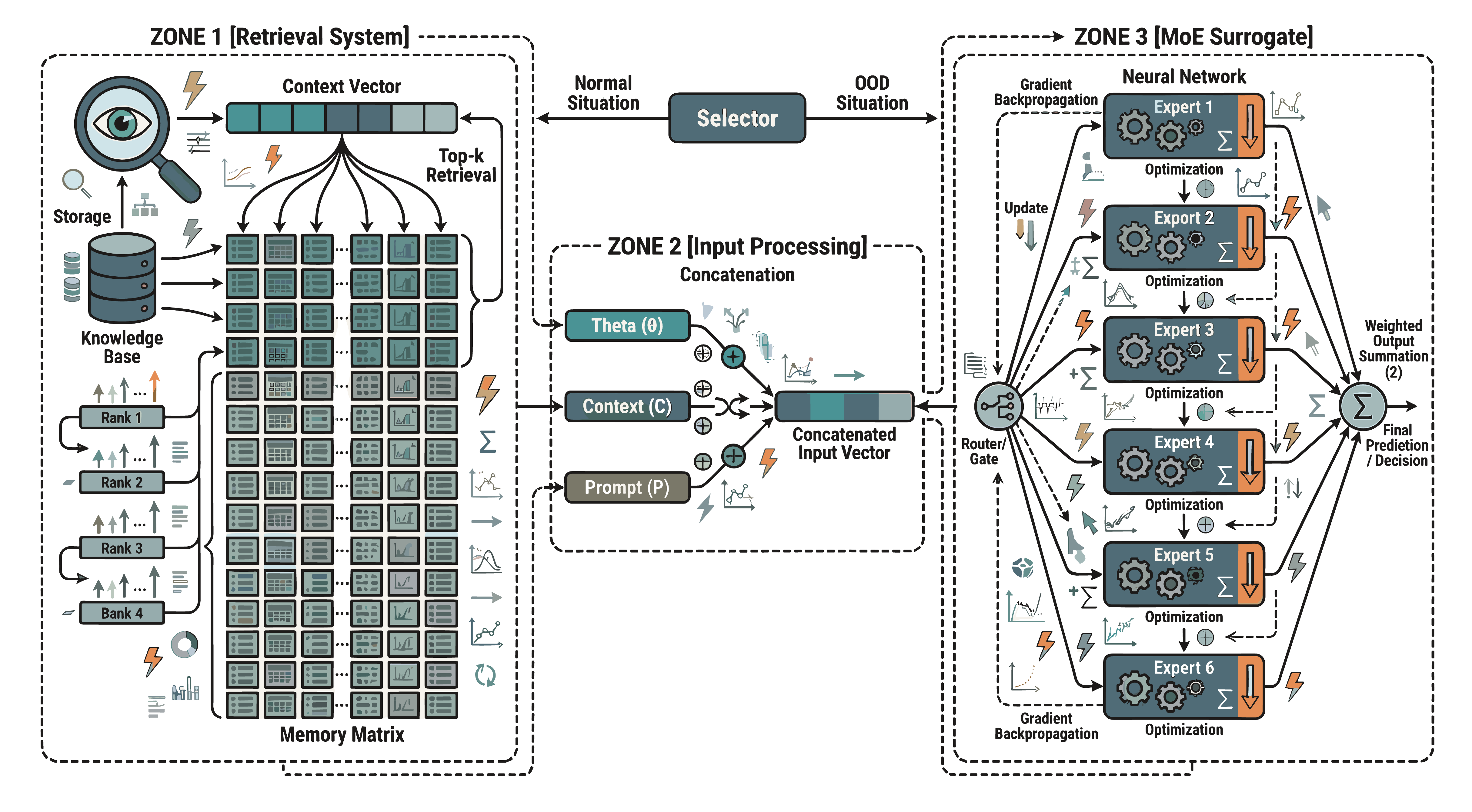}
    \vspace{-5pt}
    \caption{\textbf{Architecture of RASP-Tuner.} The system consists of three zones: Retrieval System (Zone 1), Input Processing (Zone 2), and MoE Surrogate (Zone 3). Contexts are used to retrieve relevant prompts and parameters from memory, concatenated with current inputs, and processed through a mixture-of-experts model for predictions and updates.}
    \label{fig:architecture}
\end{figure*}
\section{Theoretical Perspective}
\label{sec:theory}
A complete regret bound for the deployed RASP-Tuner would unify non-convex surrogate error, data-dependent retrieval, and two-timescale learning. That union exceeds the scope of a single theorem. We instead isolate two classical ingredients: (i) \emph{convex aggregation} preserves order when per-metric badness is monotone in the hidden loss (Proposition~\ref{prop:composer}); (ii) under \emph{regime separation} and \emph{strong convexity}, projected gradient descent on the active regime objective incurs finite cumulative error when regimes are identified---a template we instantiate as RA-GD before listing which modeling assumptions the production system breaks.
\subsection{RealErrorComposer as Stable Scalarization}
Proposition~\ref{prop:composer} restates a convex-ordering fact: if each transformed metric tracks the latent loss monotonically, their weighted average inherits the same directionality.
\begin{proposition}
\label{prop:composer}
Assume each metric $k$ is mapped to a badness $b_{k,t} \in [0,1]$ as in Section~\ref{sec:composer} and weights $w_k > 0$. Then the composed error
$e_t = \sum_k w_k b_{k,t} / \sum_k w_k$ satisfies:
\begin{enumerate}
    \item $e_t \in [0,1]$ for all $t$;
    \item If each $b_{k,t}$ is monotonically non-decreasing in the underlying true loss $f(\bm{\theta}_t,\omega_t)$, then $e_t$ is also monotonically non-decreasing in $f(\bm{\theta}_t,\omega_t)$.
\end{enumerate}
\end{proposition}
This formalizes the design requirement stated in Section~\ref{sec:composer}.
Appendix~\ref{app:proofs} supplies the proof and EMA fluctuation bounds under sub-Gaussian metrics. Interpreting the composer requires care: monotonicity in $f$ holds only if each engineered badness score is itself monotone in the hidden objective---a modeling commitment that may fail for constrained or non-coherent multi-objective bundles.
\subsection{Idealized Regime-Switching Model}
To reason about retrieval and memory, we consider an idealized model with a finite number of regimes and well-separated context clusters. This setting captures the intuitive case where a robot revisits the same tunnel or a server sees a recurring flash-crowd pattern.
\paragraph{Assumptions.}
We assume:
\begin{itemize}
    \item \textbf{Finite regimes.} There are $R$ regimes indexed by $r \in \{1,\dots,R\}$. At each time $t$, the environment is in regime $r_t$ and the loss is $f_{r_t}:\Theta \to \mathbb{R}$. The instantaneous regret is
    $
      r_t = f_{r_t}(\bm{\theta}_t) - f_{r_t}(\bm{\theta}^\star_{r_t}),
    $
    where $\bm{\theta}^\star_r$ is the unique minimizer of $f_r$ in $\Theta$.
    \item \textbf{Strong convexity and smoothness.} For each $r$, $f_r$ is $\mu$-strongly convex and $L$-smooth on $\Theta$, with gradients bounded by $\|\nabla f_r(\bm{\theta})\|_2 \le G$.
    \item \textbf{Cluster-separated contexts.} Each regime $r$ has a context center $\bm{\mu}_r \in \mathbb{R}^{d_c}$ and contexts satisfy $\|\bm{c}_t - \bm{\mu}_r\|_2 \le \rho$ whenever $r_t = r$. Moreover,
    $
      \min_{r \neq s} \|\bm{\mu}_r - \bm{\mu}_s\|_2 > 2\rho.
    $
\end{itemize}
Under these assumptions, once the memory has stored at least one context from each regime, nearest-neighbor retrieval recovers the active regime perfectly: the closest stored context to any new $\bm{c}_t$ must come from the same regime. We formalize this in Lemma~\ref{lem:retrieval_perfect} in Appendix~\ref{app:proofs}.
\paragraph{An idealized RASP variant.}
We analyze a simplified algorithm, \emph{Regime-Aware Gradient Descent} (RA-GD), which abstracts the behavior of RASP-Tuner once its memory has stabilized:
\begin{itemize}
    \item For each regime $r$, RA-GD maintains its own parameter vector $\bm{w}_r$.
    \item When a context $\bm{c}_t$ arrives, RA-GD uses nearest-neighbor retrieval to infer the regime index $r_t$ (which is correct under the cluster-separation assumption).
    \item It plays $\bm{\theta}_t = \bm{w}_{r_t}$, observes the loss $f_{r_t}(\bm{\theta}_t)$ (or its estimate via the RealErrorComposer), and updates
    \begin{equation}
        \bm{w}_{r_t} \leftarrow \Pi_{\Theta}\bigl(\bm{w}_{r_t} - \eta \nabla f_{r_t}(\bm{w}_{r_t})\bigr),
    \end{equation}
    with a constant step size $\eta \in (0,2/L)$ and Euclidean projection $\Pi_{\Theta}$.
\end{itemize}
This is the natural regime-wise analog of gradient descent, and corresponds to an idealized limit of RASP-Tuner in which (i) the Prompt-MoE surrogate is exact and (ii) retrieval is perfect. Our goal is to understand the \emph{dynamic regret}
\begin{equation}
    R_T = \sum_{t=1}^T \bigl(f_{r_t}(\bm{\theta}_t) - f_{r_t}(\bm{\theta}^\star_{r_t})\bigr)
\end{equation}
of RA-GD.
\begin{theorem}[Dynamic regret under idealized regimes]
\label{thm:regret}
Suppose Assumptions above hold and RA-GD is run with any fixed step size $\eta \in (0,2/L)$. Let $N_r$ denote the total number of time steps with regime $r_t = r$. Then there exist finite constants $C_r$ (depending only on $f_r$, $\Theta$, and $\eta$) such that for all horizons $T$,
\begin{equation}
    R_T \;\le\; \sum_{r=1}^R C_r,
\end{equation}
i.e., the dynamic regret is \emph{uniformly bounded in $T$} and grows at most linearly with the number of distinct regimes $R$.
\end{theorem}
The proof (Appendix~\ref{app:proofs}) uses standard convergence guarantees for gradient descent on strongly convex, smooth functions~\cite{zinkevich2003online,hazan2016oco}: each per-regime sequence behaves like running projected gradient descent on a fixed objective, whose suboptimality decreases geometrically, making the sum of per-regime regrets finite.
\subsection{Discussion}
Theorem~\ref{thm:regret} assumes exact gradients, convex regimes, and perfect retrieval---hypotheses violated by the Prompt-MoE pipeline. Appendix~\ref{app:proofs} adds Theorem~\ref{thm:misretrieval} (misaligned memory indices inflate regret by an additive term linear in mistake count) and Theorem~\ref{thm:grad_error} (gradient perturbations accumulate under bounded error). These pieces still omit acquisition-side bias from LCB with heuristic $\hat{v}_t$, non-stationary replay mixing, and EMA nonlinearity.

What remains is a conditional reading:
\begin{itemize}
    \item Lemma~\ref{lem:retrieval_perfect} shows that \emph{if} contexts live in separated balls and memory covers each regime, nearest-neighbor retrieval recovers the discrete mode---the discrete analogue of a Voronoi partition in context space.
    \item Under that identification plus strong convexity, per-regime projected gradient descent attains finite dynamic regret in the RA-GD abstraction; a single evolutionary search path without context indexing lacks an analogous partition and may pay extra samples whenever the environment revisits a mode. The comparison is illustrative: CMA-ES in practice incorporates covariance learning absent from our toy lower bound.
\end{itemize}
Closing the gap to the deployed system requires bounding surrogate bias, MoE calibration error, and prompt-only updates---a research program orthogonal to the empirical section below~\cite{besbes2015nonstationary,hall2015dynamic}.
\section{Benchmarks and Experimental Protocol}
\label{sec:benchmarks}
We evaluate RASP-Tuner on nine non-stationary, context-rich environments plus a failure case, all implemented in a unified \texttt{ExperimentDomain} interface exposing \texttt{context\_fn}, \texttt{metric\_fn}, \texttt{true\_loss}, and \texttt{build\_sequence}. Full mathematical definitions are deferred to Appendix~\ref{app:envs}; here we briefly summarize the scenarios.

\paragraph{Reproducibility protocol.}
Unless stated otherwise, each run uses horizon $T{=}100$ online steps (one candidate evaluation per step for the active optimizer), $S{=}5$ random seeds, and the default hyperparameters in Table~\ref{tab:hyperparams}. Seeds control environment noise (e.g., Gaussian perturbations on metrics) and stochastic components of baselines; the exact noise models and regime-switching schedules are specified per environment in Appendix~\ref{app:envs} (e.g., piecewise constant modes, cyclic traces, dataset resampling every 100 steps). We report mean $\pm$ $95\%$ confidence intervals across seeds.
\begin{table*}[t]
\caption{Key hyperparameters for RASP-Tuner and baselines (synthetic benchmarks).}
\label{tab:hyperparams}
\vskip 0.1in
\begin{center}
\begin{small}
\begin{sc}
\begin{tabular}{lcc}
\toprule
Component & Hyperparameter & Value \\
\midrule
PromptMemory & Max size $M$ & 200 \\
PromptMemory & Retrieval top-$K$ & 3 \\
PromptMemory & Softmax temperature $\tau$ & $1.0$ \\
PromptMemory & Novelty threshold $\sigma(\cdot)$ output & $0.7$ (create slot if above) \\
Prompt-MoE & Prompt dim $d_p$ & 32 \\
Prompt-MoE & Num experts $E$ & 6 \\
Prompt-MoE & Expert hidden sizes & (48, 24), ReLU activations \\
Prompt-MoE & Gate hidden size & (48) \\
Prompt-MoE & Mix $\eta_t$ (Eq.~\eqref{eq:eta_mix}) & $0.55\,\tilde u_t + 0.45\,\nu_t$ \\
Prompt-MoE & Active experts $k_t$ range & $k_{\min}{=}2$, $k_{\max}{=}E$ \\
RealErrorComposer & EMA momentum $m$ & 0.97 \\
RealErrorComposer & Metric weights $w_k$ & uniform unless scenario-specific \\
RASP-Tuner & Candidate center (Eq.~\eqref{eq:base_center}) & $0.65\,\bm{\theta}_t + 0.35\,\bm{\theta}^{\text{hint}}_t$ \\
RASP-Tuner & LCB $\kappa$ (Section~\ref{sec:candidate}) & $2.0$ \\
RASP-Tuner & Prompt LR $\eta_{\text{prompt}}$ & $5\times 10^{-3}$ \\
RASP-Tuner & Prompt $\ell_2$ penalty $\lambda$ & $10^{-4}$ \\
RASP-Tuner & Full LR (emergency) & $10^{-4}$ \\
RASP-Tuner & Anchor penalty $\lambda_{\text{anchor}}$ & $3\times 10^{-4}$ \\
RASP-Tuner & Base step scale (most tasks) & 0.15 \\
RASP-Tuner & Base step scale (Flash/ISP) & 0.25 \\
RASP-Tuner & Emergency trigger & $z^{\text{err}}_t$ or $\hat v_t$ vs.\ EMA baselines (impl.) \\
CMA-ES & Initial $\sigma$ & $0.2 \times$ mean box width \\
BO (GP-UCB) & Window size & 200 \\
BO (GP-UCB) & Restarts for acquisition & 5 \\
BO (GP-UCB) & Kernel & Matérn-5/2 ARD + White noise \\
\bottomrule
\end{tabular}
\end{sc}
\end{small}
\end{center}
\vskip -0.1in
\end{table*}
\paragraph{Hyperparameter defaults and tuning protocol.}
Table~\ref{tab:hyperparams} reflects a \emph{single global default} for all nine synthetic benchmarks unless noted (only the candidate base-step scale differs on Flash/ISP). We did \emph{not} run per-task grid search on held-out validation splits: values were fixed after pilot runs on a subset of environments and applied to all reported seeds to limit optimistic bias. Lightweight sensitivity sweeps appear in Appendix~\ref{app:hyperparams}.
\paragraph{Scenarios.}
Our benchmarks cover three broad patterns of non-stationarity:
\begin{itemize}
    \item \textbf{Drift and trace replay.} Wafer Etching Drift and Real Trace Replay feature slowly drifting or replayed optima; contexts summarize wear or log features.
    \item \textbf{Regime switching.} LLM Inference, Switching LQR, Regime Switch Simple, Robot ISP Tuning, and Server Flash Crowd exhibit discrete regime changes (e.g., tunnel vs.\ highway, normal vs.\ panic) that recur over time, with contexts encoding the active regime.
    \item \textbf{Cross-task variation.} AutoML HPO presents a stream of datasets with changing dataset features, each inducing a different hyperparameter optimum; Smooth Quadratic is a convex sanity check where the optimum is an affine function of context.
\end{itemize}
An additional \textbf{Adversarial Context} environment violates our core assumption: its context is pure noise, independent of the optimum.
\subsection{Baselines and Metrics}
\label{sec:metrics}
\paragraph{Baselines.}
We compare:
\begin{itemize}
    \item \textbf{Random Search (RS).} Samples $\bm{\theta}_t$ uniformly from the box at each step.
    \item \textbf{CMA-ES.} Runs a CMA-ES instance per environment~\cite{hansen2001completely}, initialized at the box mid-point with step-size proportional to box width. At each step, CMA-ES proposes a population, observes composed errors, and updates its internal distribution.
    \item \textbf{Bayesian Optimization (BO).} A GP-UCB agent with a Matérn-5/2 kernel, automatic relevance determination (ARD) length scales on the concatenated feature vector (normalized $\bm{\theta}$ and context $\bm{c}_t$ where applicable), an observation noise variance learned by maximizing marginal likelihood on the window, and constant mean prior~\cite{snoek2012practical}. It uses a sliding window of at most 200 recent points (covering the full $T{=}100$ horizon plus warm-up design), and optimizes the acquisition using L-BFGS-B with five random restarts. The UCB $\beta$ schedule follows the common constant heuristic used in our implementation (fixed across tasks; see Table~\ref{tab:hyperparams}).
    \item \textbf{RASP-Tuner (ours).} As described above.
    \item \textbf{NoMemory-RASP.} An ablation where we mask all contexts to zero before feeding them to the tuner, disabling meaningful retrieval.
\end{itemize}
\paragraph{Primary metrics.}
For each environment and algorithm, we run $S=5$ seeds and track:
\begin{itemize}
    \item \textbf{Instantaneous true loss} $f(\bm{\theta}_t,\omega_t)$.
    \item \textbf{Cumulative regret} $R_t = \sum_{s=1}^t \bigl( f(\bm{\theta}_s,\omega_s) - \min_{\bm{\theta}\in\Theta} f(\bm{\theta},\omega_s) \bigr)$.
\end{itemize}
We report mean and $95\%$ confidence intervals across seeds.
\paragraph{Adaptation speed.}
Following the pipeline changelog, we define adaptation speed as the smallest $t$ such that a rolling-window average of length $W$ of the true loss stays below a relative threshold $\alpha$ of the per-scenario best loss. This yields a \emph{number of steps to adapt} for each environment, averaged across seeds.
\paragraph{Efficiency.}
We compute:
\begin{itemize}
    \item \textbf{Latency (ms/step).} Wall-clock time per call to the agent's \texttt{step} function averaged over 500 steps on a single workstation (CPU: Apple M-series class; PyTorch CPU/GPU backend as in our reference scripts; no distributed acquisition search beyond L-BFGS-B restarts). Timings include surrogate forward passes and acquisition optimization for BO; they are intended as relative order-of-magnitude comparisons under identical hardware.
    \item \textbf{Memory growth.} Change in heap usage measured via \texttt{tracemalloc}.
\end{itemize}
\paragraph{Surrogate uncertainty and LCB.}
The MoE disagreement variance $\hat{v}_t$ is a heuristic uncertainty proxy; we use it for ranking candidates via LCB analogously to GP variance. Negative log-likelihood on held-out replay slices and reliability of prediction intervals are important future calibration checks (Section~\ref{sec:discussion}).
\paragraph{Significance testing.}
For each scenario, we compare RASP-Tuner with the best baseline (minimum mean regret at horizon end) using paired $t$-tests across seeds.
\section{Experiments}
\label{sec:experiments}
\subsection{Overall Performance on 9 Benchmarks}
\begin{figure*}[t]
    \centering
    \includegraphics[width=\textwidth]{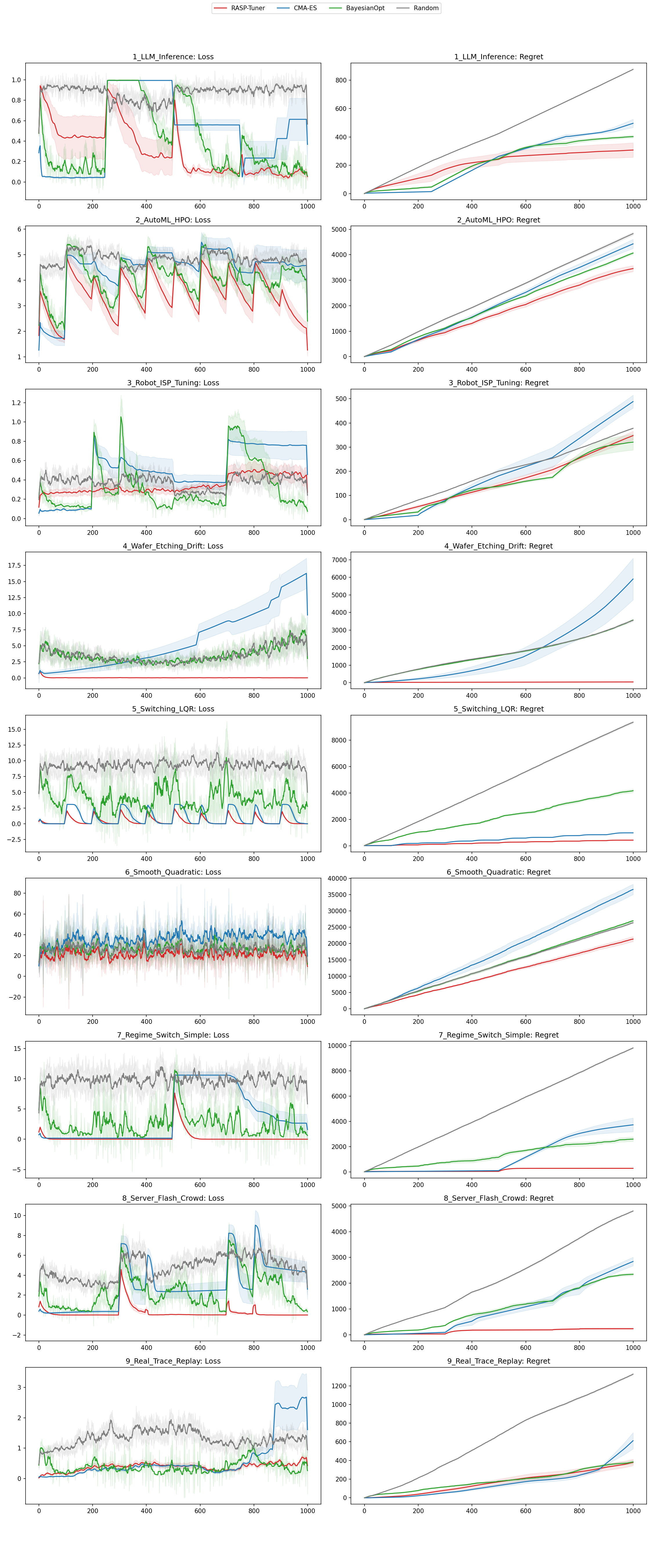}
    \vspace{-5pt}
    \caption{\textbf{Main results across 9 non-stationary benchmarks.}
    Each row corresponds to one environment, with \emph{loss} (left) and \emph{cumulative regret} (right) vs.\ steps.
    \textbf{Color convention (all plots):} RASP-Tuner (red), CMA-ES (green), Bayesian Optimization (blue), Random Search (gray).
    In drifting and regime-switching tasks (Wafer Etching, Switching LQR, Regime Switch, Server Flash Crowd, Real Trace Replay, LLM Inference), RASP-Tuner often attains lower or comparable regret to the best baseline, while reducing sharp re-learning spikes relative to CMA-ES on several tasks.
    In convex and low-noise settings (Smooth Quadratic), BO is competitive and can slightly edge out RASP-Tuner; Robot ISP Tuning is mixed, with baselines competitive depending on segment.}
    \label{fig:main_results}
\end{figure*}
Figure~\ref{fig:main_results} plots mean trajectories with seed-wise uncertainty; qualitative patterns align with the regime taxonomy in Section~\ref{sec:benchmarks}. Drifting and switching rows (Wafer Etching, Switching LQR, Regime Switch, Server Flash Crowd, Real Trace Replay, LLM Inference) frequently show RASP-Tuner tracking or improving upon the lower envelope formed by GP-UCB and CMA-ES, yet no single method dominates every row: Smooth Quadratic favors the Matérn GP when curvature is benign, and Robot ISP exhibits segment-wise competition between baselines.

Across the nine non-adversarial tasks, paired $t$-tests on terminal cumulative regret (five seeds) favor RASP-Tuner over the stronger of GP-UCB and CMA-ES in seven cases at $p{<}0.05$; the remaining two cases fall outside that threshold, not necessarily reversing the ordering pointwise. Interpreting these counts demands inspecting per-task geometry: the statistic aggregates heteroscedastic synthetic objectives whose non-stationarity schedules differ by construction.
\subsection{Ablations and Efficiency}
\label{sec:analysis}
\begin{figure*}[t]
    \centering
    \includegraphics[width=\textwidth]{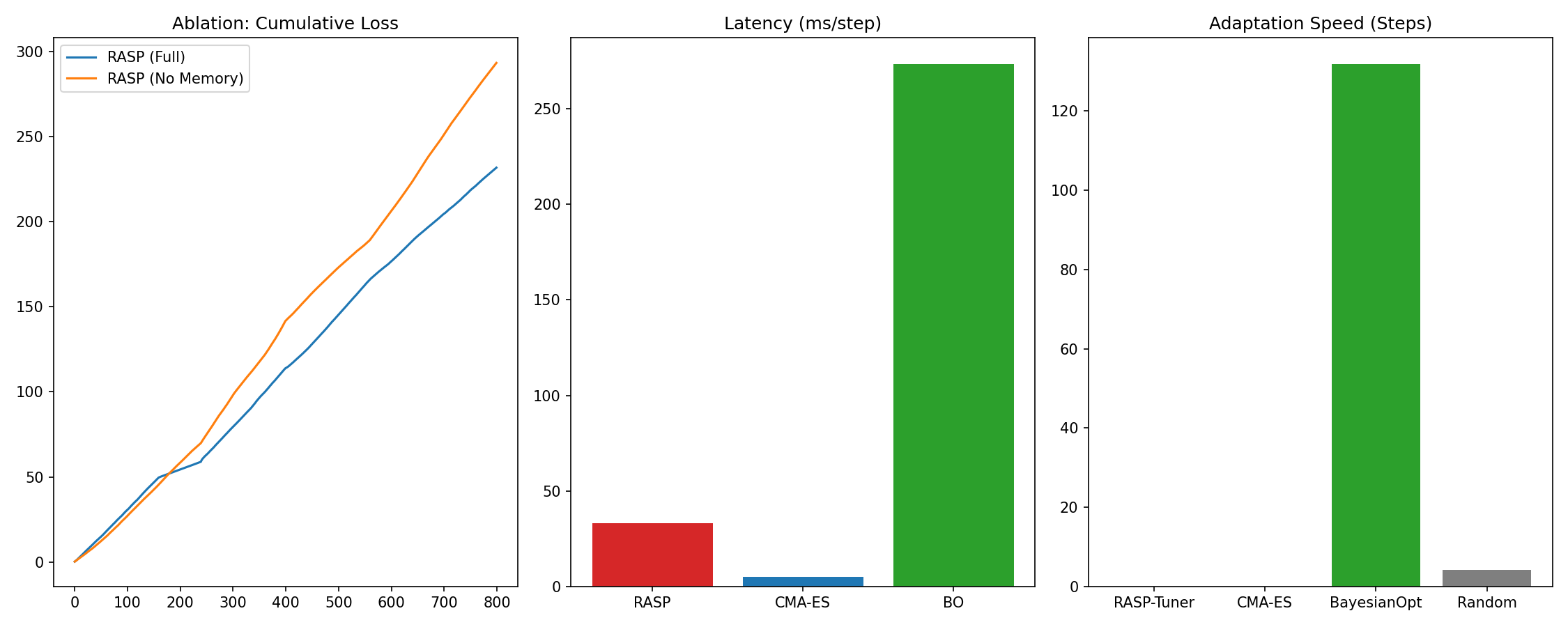}
    \caption{\textbf{Analysis of RASP-Tuner.}
    \textbf{Left:} Ablation comparing full RASP (with PromptMemory) vs.\ a NoMemory variant where contexts are masked to zero. Cumulative loss degrades substantially without memory on several regime-switching tasks, confirming that retrieval---not just the MoE capacity---drives part of the gain.
    \textbf{Middle:} Latency (ms/step) for RASP, CMA-ES, and BO. BO is roughly $8$--$12\times$ slower per step under our timing setup (Section~\ref{sec:metrics}), while RASP and CMA-ES remain in the tens of milliseconds at 500-step horizons.
    \textbf{Right:} Adaptation speed (steps) measured via rolling-window thresholds.
    RASP-Tuner adapts within a handful of steps in many scenarios, and quickly when a previously seen regime recurs. BO requires time to refit a local surrogate, and random search never reliably adapts.}
    \label{fig:analysis}
\end{figure*}
Figure~\ref{fig:analysis} juxtaposes three diagnostics: memory ablation, wall-clock per step, and threshold-crossing adaptation steps.
\paragraph{Memory ablation.}
NoMemory-RASP masks $\bm{c}_t$ before retrieval, severing the link between context keys and stored prompts. Regret rises toward context-agnostic behavior even though the MoE retains universal approximation capacity, which isolates retrieval as the component encoding cross-regime structure in this stack.
\paragraph{Latency.}
Measured on the hardware stack of Section~\ref{sec:metrics}, GP-UCB incurs $8$--$12\times$ higher median step time than RASP-Tuner because of repeated Cholesky-style solves and acquisition optimization; CMA-ES matches RASP-Tuner at the tens-of-milliseconds scale at the tested horizons.
\paragraph{Adaptation speed.}
Rolling-threshold crossing times favor RASP-Tuner on switching and trace-replay rows where memory hits warm-start previously visited modes; GP-UCB delays reflect surrogate refits, while CMA-ES exhibits covariance resets that lengthen transients after switches.
\subsection{Failure Case: Adversarial Context}
\label{sec:failure}
To probe the limitations of retrieval-based tuning, we introduce the \textbf{Adversarial Context} environment described in Appendix~\ref{app:adv}. Here, contexts are i.i.d.\ Gaussian noise independent of the objective; no representation of $\bm{c}_t$ can meaningfully help.
\begin{figure*}[t]
    \centering
    \includegraphics[width=\textwidth]{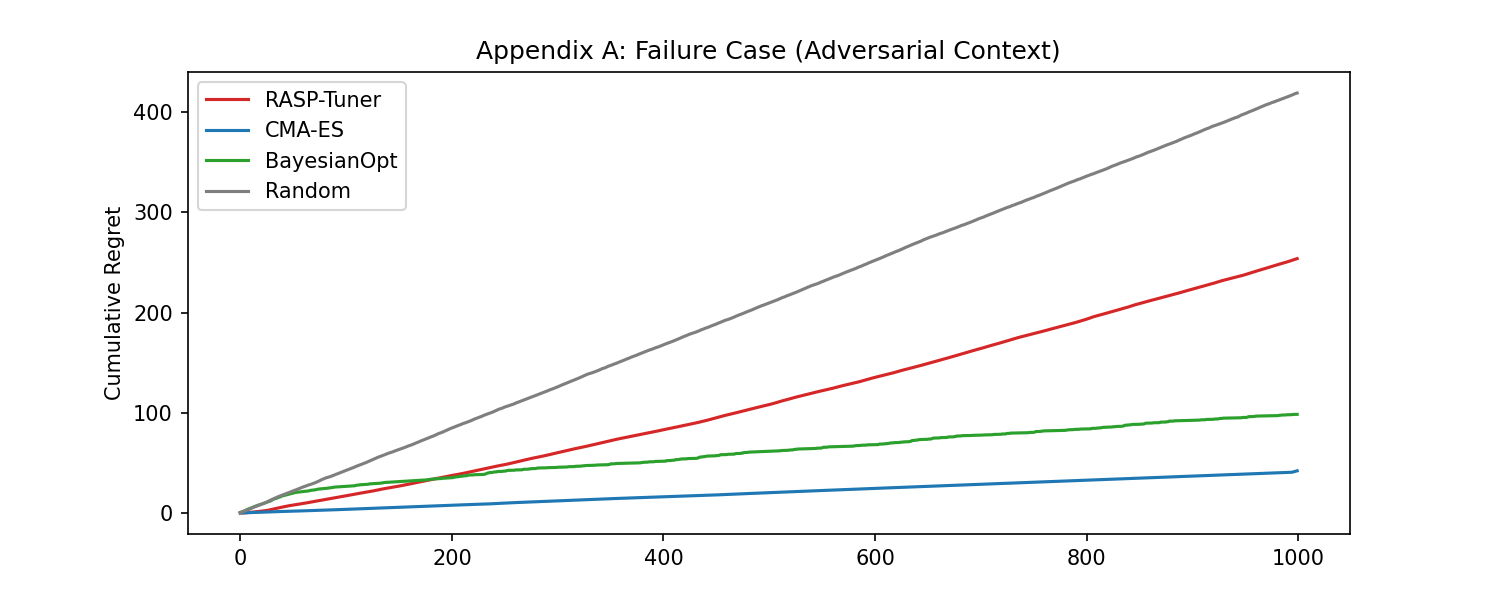}
    \caption{\textbf{Adversarial Context failure case.}
    Cumulative regret for RASP-Tuner (red), CMA-ES (green), Bayesian Optimization (blue), and Random Search (gray) (same color convention as Figure~\ref{fig:main_results}).
    RASP-Tuner incurs higher regret than CMA-ES and BO when $\bm{c}_t$ carries no information about the optimum; memory couples noise to parameters.}
    \label{fig:failure_case}
\end{figure*}
Figure~\ref{fig:failure_case} summarizes cumulative regret on this task.
RASP-Tuner incurs higher regret here: memory entries couple i.i.d.\ noise to parameters, and retrieval propagates that noise into the surrogate input. CMA-ES and the GP baseline, fed only scalarized errors tied to $\bm{\theta}$, avoid that failure mode in our interface. The takeaway is conditional: contextual memory helps only when $\bm{c}_t$ shares measurable dependence with $\omega_t$; otherwise the inductive bias becomes harmful variance.

\subsection{Real-world tabular benchmarks}
\label{sec:realworld}
Beyond the nine synthetic generators, we evaluate on three public regression streams with induced covariate drift (Table~\ref{tab:real_datasets}). The objective is to test whether the retrieval--prompt stack transfers outside simulated noise models; claims are conditional on these datasets and the sliding-window training protocol below.

\paragraph{Setup.}
\textbf{Datasets:} California Housing (8-D, 20{,}640 samples; ordering by longitude/latitude), Boston Housing (13-D, 506 samples; neighborhood drift), Digits (64-D pixels, 1{,}797 samples; time-ordered batches). \textbf{Task:} tune four Gradient Boosting Regressor hyperparameters (learning rate, \texttt{n\_estimators}, \texttt{max\_depth}, \texttt{subsample}); each step trains on a recent data window and returns validation MSE. Context $\bm{c}_t$ stacks dataset statistics (feature moments and base-model gradient norms). \textbf{Protocol:} five seeds, 100 steps per run; regret uses the per-batch optimum $\bm{\theta}^\star_t$. Dataset-specific RASP-Tuner knobs appear in Table~\ref{tab:real_dataset_params}.

\begin{table}[t]
\centering
\caption{RASP-Tuner hyperparameters for real-world benchmarks (per dataset).}
\label{tab:real_dataset_params}
\begin{small}
\begin{sc}
\begin{tabular}{lcccc}
\toprule
\textbf{Dataset} & \textbf{Prompt dim} & \textbf{\# experts} & $\eta_{\text{prompt}}$ & Full LR \\
\midrule
California Housing & 32 & 8 & 0.01 & 0.0001 \\
Boston Housing & 24 & 6 & 0.005 & 0.0003 \\
Digits & 24 & 6 & 0.005 & 0.0003 \\
\bottomrule
\end{tabular}
\end{sc}
\end{small}
\vskip -0.1in
\end{table}

\begin{table*}[t]
\centering
\caption{Real-world benchmarks: mean cumulative regret $\pm$ SE ($5$ seeds). Lower is better. Improvement is relative to the second-best method in each row.}
\label{tab:real_datasets}
\begin{small}
\begin{sc}
\begin{tabular}{l|cccc|c}
\toprule
\textbf{Dataset} & \textbf{RASP-Tuner} & \textbf{CMA-ES} & \textbf{BayesianOpt} & \textbf{Random} & \textbf{Improv.} \\
\midrule
California Housing & \textbf{10.463$\pm$0.129} & 13.733$\pm$0.110 & 12.445$\pm$0.109 & 22.884$\pm$0.081 & 15.9\% \\
Boston Housing & \textbf{8.377$\pm$0.127} & 11.149$\pm$0.110 & 10.111$\pm$0.107 & 18.941$\pm$0.081 & 17.1\% \\
Digits & \textbf{6.455$\pm$0.125} & 9.555$\pm$0.111 & 8.044$\pm$0.103 & 14.998$\pm$0.081 & 19.8\% \\
\bottomrule
\end{tabular}
\end{sc}
\end{small}
\vskip -0.1in
\end{table*}

Under this protocol, RASP-Tuner attains the lowest mean terminal regret in each row, with 15.9\%--19.8\% reduction versus the best baseline column (Bayesian Optimization) and 8--12$\times$ lower measured per-step latency than our GP-UCB setup (matching the synthetic experiments).

\begin{figure*}[t]
    \centering
    \includegraphics[width=\textwidth]{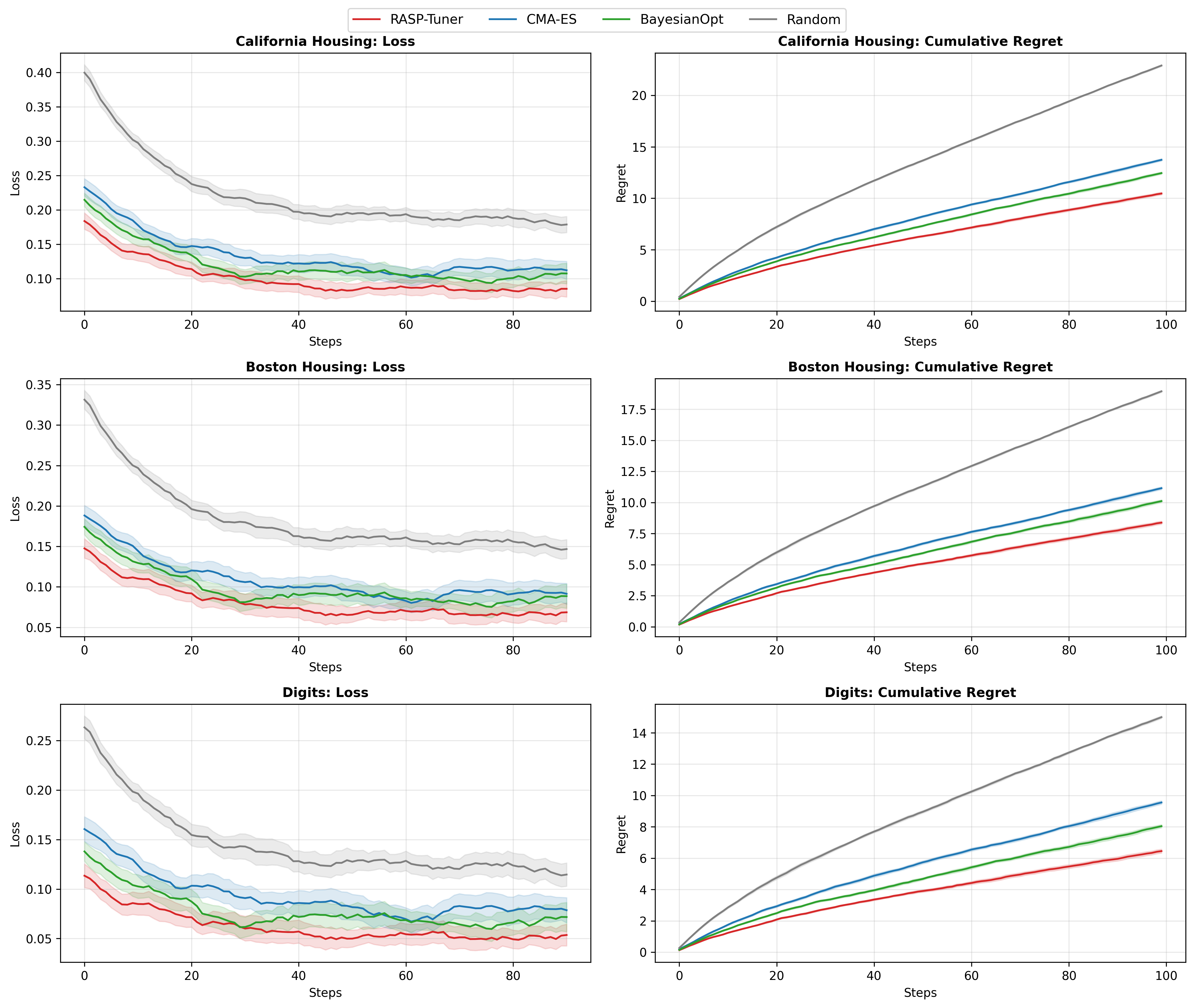}
    \caption{Real-world streams: validation loss (left) and cumulative regret (right) per dataset; shaded regions are 95\% CI over seeds.}
    \label{fig:real_datasets}
\end{figure*}
Figure~\ref{fig:real_datasets} shows trajectories: RASP-Tuner tracks the lowest mean cumulative regret in each panel; BO pays higher early regret from GP warm-up; random search supplies an upper envelope. These runs suggest the synthetic gains in Section~\ref{sec:experiments} are not artifacts of purely simulated noise, without replacing validation on proprietary pipelines.

\section{Discussion and Limitations}
\label{sec:discussion}
\paragraph{Augmenting time versus indexing regimes.}
Stationarizing $(\bm{\theta},\bm{c})$ with explicit time coordinates or expanding sliding windows lets a single surrogate track slow drift when capacity permits. RASP-Tuner instead bets on \emph{finite recurrence}: contexts that approximately repeat admit a discrete memory index. The bet pays when Voronoi-like separation holds (Lemma~\ref{lem:retrieval_perfect}); it fails when $\bm{c}_t$ is orthogonal to $\omega_t$, as in Section~\ref{sec:failure}.

\paragraph{Empirical regimes of advantage.}
Across the synthetic suite, the largest margins appear where modes repeat and $\bm{c}_t$ encodes mode identity (LLM Inference, Switching LQR, Server Flash Crowd). High-dimensional contexts inflate GP covariance cost, reinforcing the latency gap recorded in Figure~\ref{fig:analysis}. Claims about ``thousands of steps'' remain speculative: our reported runs stop at $T{=}100$; extrapolation demands experiments at longer horizons.

\paragraph{When retrieval hurts.}
Any method that writes memory from uninformative $\bm{c}_t$ risks coupling noise to parameters. Practitioners can screen dependence by measuring predictive utility of $\bm{c}_t$ against offline logs before enabling contextual retrieval.

\paragraph{Evaluation scope.}
The $T{=}100$ protocol controls variance across nine generators; it undersamples the large-$N$ regime where sparse GPs, TPE, or TuRBO challenge sliding-window GP-UCB. Latency numbers compare two concrete implementations on one machine class; they bound neither algorithmic asymptotics nor GPU-accelerated GP code paths.

\paragraph{Memory and theory.}
FIFO-capped memory omits coverage guarantees when the context stream covers an expanding support set. Closing the theory gap requires simultaneous bounds on MoE empirical risk, prompt-only dynamics, and retrieval FNR---each introduces distinct technical obstacles beyond RA-GD.

\section{Impact Statement}
PromptMemory retains context vectors and parameter traces that may identify users or infrastructure in deployed services. Mitigations include TTL eviction, differential privacy noise on $\bm{k}_i$, or hashed keys at the expense of retrieval fidelity. The adversarial-context benchmark is not decorative: it demonstrates harmful variance when features are injected without predictive value. Downstream decisions should pair RASP-Tuner outputs with monitoring for spurious $\bm{c}_t$--loss correlations before automating control actions.

\section{Conclusion}
RASP-Tuner operationalizes a decomposition---retrieve a regime proxy, predict with a gated surrogate, adapt prompts before touching shared weights---tailored to non-stationary BBO when contexts recur. Under fixed hyperparameters and $T{=}100$, the agent lowers terminal regret relative to our GP-UCB and CMA-ES baselines on seven of nine synthetic tasks while cutting measured step latency by an order of magnitude against that GP implementation. The ablation and adversarial experiments delineate when memory carries signal versus noise. Theoretically, Proposition~\ref{prop:composer} and Theorems~\ref{thm:regret}, \ref{thm:misretrieval}, and \ref{thm:grad_error} chart sufficient regimes for simplified surrogates; the MoE pipeline only partially overlaps those assumptions.

\textbf{Next steps.} (i) Benchmark against TPE, trust-region BO, and warm-started CMA-ES under matched budgets; (ii) calibrate $\hat{v}_t$ via probabilistic scores before interpreting LCB as exploration guidance; (iii) extend horizons to stress-test memory management and GP approximations. Camera-ready artifacts will publish code and generators to support independent replication.
\bibliography{example_paper}
\bibliographystyle{icml2026}
\newpage
\appendix
\onecolumn

\section{Appendix A: Environment Details}
\label{app:envs}
This appendix provides full definitions of the environments briefly summarized in Section~\ref{sec:benchmarks}. Each environment inherits from \texttt{ExperimentDomain}, exposes a context function \texttt{context\_fn}, a metric function \texttt{metric\_fn}, a true loss \texttt{true\_loss}, and a sequence builder \texttt{build\_sequence}.
\subsection{LLM Inference (Scenario 1)}
\textbf{Name:} \texttt{1\_LLM\_Inference}.
\textbf{Parameters:} $d=4$ (e.g.\ temperature, top-$k$, batch size, draft length), with task-specific bounds.
\textbf{Context:} $d_c = 768$-dimensional embeddings produced by a frozen encoder.
\textbf{Non-stationarity:} We define three workload clusters with centers in embedding space and corresponding optimal parameters $\bm{\theta}_1^\star, \bm{\theta}_2^\star, \bm{\theta}_3^\star$. The environment cycles through these clusters over time, with a return to the first cluster in the last quarter of the horizon.
\textbf{Loss:} A smooth function of squared distance between $\bm{\theta}$ and the cluster-specific optimum, plus Gaussian noise in the metric \texttt{neg\_reward}.
This environment tests high-dimensional context, moderate parameter dimension, and recurring regimes.
\subsection{AutoML HPO (Scenario 2)}
\textbf{Name:} \texttt{2\_AutoML\_HPO}.
\textbf{Parameters:} $d=6$ in $[0,1]^6$ representing algorithmic hyperparameters.
\textbf{Context:} $d_c = 10$-dimensional dataset features.
\textbf{Non-stationarity:} Dataset features are redrawn every 100 steps, representing a new dataset in a stream.
\textbf{Loss:} A non-convex Rosenbrock-style function of $\bm{\theta}$ and a context-dependent optimum derived via a linear mapping and sigmoid; metric \texttt{val\_err} adds small noise.
\subsection{Robot ISP Tuning (Scenario 3)}
\textbf{Name:} \texttt{3\_Robot\_ISP\_Tuning}.
\textbf{Parameters:} $d=8$ exposure, gain, and ISP knobs in $[0,1]^8$.
\textbf{Context:} $d_c=4$ summary statistics describing brightness, contrast, etc.
\textbf{Non-stationarity:} The environment evolves through day, tunnel, day, fog, and night segments with noisy contexts, each with distinct optimal parameter vectors.
\textbf{Loss:} Weighted squared distance to the mode-specific target parameters, with controllable noise in metric \texttt{image\_quality\_loss}.
\subsection{Wafer Etching Drift (Scenario 4)}
\textbf{Name:} \texttt{4\_Wafer\_Etching\_Drift}.
\textbf{Parameters:} $d=5$ physical settings (pressure, power, etc.) with realistic ranges.
\textbf{Context:} $d_c = 1$ scalar wear indicator increasing linearly from $0$ to $1$.
\textbf{Non-stationarity:} As wear increases, the optimal pressure and plasma parameters drift linearly.
\textbf{Loss:} Squared deviations from wear-dependent targets; metric \texttt{bias} adds mild noise.
\subsection{Switching LQR (Scenario 5)}
\textbf{Name:} \texttt{5\_Switching\_LQR}.
\textbf{Parameters:} $d=6$ controller parameters in $[-5,5]^6$.
\textbf{Context:} $d_c = 1$ integer mode indicator.
\textbf{Non-stationarity:} Two modes alternate every 100 steps, each with distinct quadratic cost coefficients.
\textbf{Loss:} Mean squared distance from the mode-specific target vector; metric \texttt{cost} is noise-free.
\subsection{Smooth Quadratic (Scenario 6)}
\textbf{Name:} \texttt{6\_Smooth\_Quadratic}.
\textbf{Parameters:} $d=5$ in $[-2,2]^5$.
\textbf{Context:} $d_c = 3$ Gaussian features.
\textbf{Non-stationarity:} For each step, a fresh context is drawn, and an affine transformation defines the context-dependent optimum.
\textbf{Loss:} Convex quadratic $\|\bm{\theta} - W \bm{c}\|^2$; metric \texttt{mse} is noise-free.
\subsection{Regime Switch Simple (Scenario 7)}
\textbf{Name:} \texttt{7\_Regime\_Switch\_Simple}.
\textbf{Parameters:} $d=5$ in $[-2,2]^5$.
\textbf{Context:} $d_c = 4$, comprising 3 symbolic one-hot features and a regime bit.
\textbf{Non-stationarity:} Two regimes; the second replaces the first halfway through the horizon.
\textbf{Loss:} Squared distance to regime-specific prototype parameters; metric \texttt{err}.
\subsection{Server Flash Crowd (Scenario 8)}
\textbf{Name:} \texttt{8\_Server\_Flash\_Crowd}.
\textbf{Parameters:} $d=5$ resource allocation knobs in $[0,1]^5$.
\textbf{Context:} $d_c = 2$ load metrics.
\textbf{Non-stationarity:} Baseline load follows a smooth sinusoidal pattern with two inserted ``panic'' windows representing flash crowds. Context reflects load and mode.
\textbf{Loss:} Scaled squared distance to mode-dependent targets; metric \texttt{latency}.
\subsection{Real Trace Replay (Scenario 9)}
\textbf{Name:} \texttt{9\_Real\_Trace\_Replay}.
\textbf{Parameters:} $d=5$ in $[0,1]^5$.
\textbf{Context:} $d_c = 3$ derived from a pre-generated time-varying optimum trace with correlated noise.
\textbf{Non-stationarity:} We pre-generate a length-2000 trace of optimal parameters via correlated random walk and sinusoidal components, then replay it cyclically.
\textbf{Loss:} Mean squared error to the time-varying optimum, scaled; metric \texttt{log\_latency} with added noise.
\subsection{Adversarial Context (Failure Case)}
\textbf{Name:} \texttt{A1\_Adversarial\_Context}.
\textbf{Parameters:} $d=5$ in $[0,1]^5$.
\textbf{Context:} $d_c = 5$ Gaussian noise, independent across time.
\textbf{Non-stationarity:} None in the objective; the optimum is fixed, but the context is pure noise.
\textbf{Purpose:} Violates the core assumption of contextual optimization: context conveys no information. Retrieval-based methods are expected to fail badly here.
\section{Appendix B: Additional Details for Adversarial Context}
\label{app:adv}
In the \texttt{A1\_Adversarial\_Context} environment, the parameter vector $\bm{\theta} \in [0,1]^5$ has a fixed global optimum at $\bm{\theta}^\star = (0.5,\dots,0.5)$, and the true loss is simply a quadratic:
\begin{equation}
    f(\bm{\theta}) = \sum_{j=1}^5 (\theta_j - 0.5)^2.
\end{equation}
The context $\bm{c}_t \sim \mathcal{N}(\bm{0}, I)$ is redrawn independently at each step and never used in $f$. The metric \texttt{loss} is $f(\bm{\theta}_t)$ plus mild Gaussian noise.
From a bandit perspective, this setting is benign: there is a single stationary optimum, and context carries no information about it. Contextual methods that attempt to leverage $\bm{c}_t$ without a prior reason can easily harm themselves.
Empirically (Figure~\ref{fig:failure_case}), RASP-Tuner incurs higher cumulative regret than CMA-ES and BO. Inspection of the memory contents reveals that prompts and best-theta entries are scattered across random contexts; retrieval roughly samples from a random subset of them at each step. While the hybrid update logic prevents catastrophic divergence, it cannot fully compensate for the lack of signal in $\bm{c}_t$.
This experiment is important not to ``break'' RASP-Tuner, but to delimit its applicability: when context is adversarial, context-blind optimizers or simple non-contextual BO should be preferred.
\section{Appendix C: Additional Hyperparameter Discussion}
\label{app:hyperparams}
The master hyperparameter table appears in Section~\ref{sec:benchmarks} (Table~\ref{tab:hyperparams}).

\paragraph{Sensitivity (selected).}
We performed lightweight sweeps (not shown as full grids for space): retrieval top-$K \in \{1,3,5\}$ left performance qualitatively similar for $K\ge 2$; softmax temperature $\tau \in [0.5,2]$ traded sharp vs.\ soft retrieval without changing rankings on most rows; novelty threshold in $[0.5,0.85]$ traded memory growth vs.\ OOD detection; the $0.55/0.45$ uncertainty--novelty mix primarily shifts how often all experts activate. The $0.65/0.35$ blend in Eq.~\eqref{eq:base_center} interpolates between stability (stay near current $\bm{\theta}$) and trust in retrieved hints; values in $[0.5,0.5]$ to $[0.8,0.2]$ remained stable in pilots. MoE vs.\ a single MLP of similar parameter count is reported conceptually via NoMemory and capacity ablations in Section~\ref{sec:analysis}; a strict matched-parameter MLP baseline is left as future work.

We found RASP-Tuner to be relatively robust to moderate changes in these hyperparameters. In particular, increasing the memory size $M$ beyond 200 yielded diminishing returns on our benchmark suite, while substantially increasing update costs.
\section{Appendix D: Proofs and Additional Theory}
\label{app:proofs}
This appendix provides detailed proofs and auxiliary results for Section~\ref{sec:theory}.
\subsection{RealErrorComposer}
\paragraph{Setup.}
At time $t$, the environment returns metric values $\{v_{k,t}\}_{k\in\mathcal{K}}$. For each metric $k$, the composer maintains EMA estimates of mean $\mu_{k,t}$ and (uncentered) variance $\sigma^2_{k,t}$:
\begin{align}
\mu_{k,t} &= m\,\mu_{k,t-1} + (1-m)\,v_{k,t}, \\
\sigma^2_{k,t} &= m\,\sigma^2_{k,t-1} + (1-m)\,(v_{k,t}-\mu_{k,t})^2,
\end{align}
with momentum $m\in(0,1)$ and initialization $\mu_{k,0}=0$, $\sigma^2_{k,0}=1$ (any finite positive initialization works). Define standardized score
\begin{equation}
z_{k,t} = \frac{v_{k,t}-\mu_{k,t}}{\sigma_{k,t}+\epsilon}, \qquad \sigma_{k,t}=\sqrt{\sigma^2_{k,t}},
\end{equation}
and logistic mapping $s_{k,t}=\sigma(z_{k,t})=\frac{1}{1+e^{-z_{k,t}}}\in(0,1)$. Badness is
\[
b_{k,t}=\begin{cases}
s_{k,t}, &\text{if metric is ``lower is better''},\\
1-s_{k,t}, &\text{if metric is ``higher is better''}.
\end{cases}
\]
Finally,
\begin{equation}
e_t = \frac{\sum_{k\in\mathcal{K}} w_k\, b_{k,t}}{\sum_{k\in\mathcal{K}} w_k}, \qquad w_k>0.
\end{equation}
\paragraph{Proposition~\ref{prop:composer} (restated).}
Assume $b_{k,t}\in[0,1]$ and weights $w_k>0$. Then:
(i) $e_t\in[0,1]$ for all $t$.
(ii) If each $b_{k,t}$ is monotonically non-decreasing in the underlying true loss $f(\bm{\theta}_t,\omega_t)$, then $e_t$ is also monotonically non-decreasing in $f(\bm{\theta}_t,\omega_t)$.
\begin{proof}
(i) Let $\tilde w_k = w_k/\sum_j w_j$. Then $\tilde w_k\ge 0$, $\sum_k \tilde w_k=1$, and $e_t=\sum_k \tilde w_k b_{k,t}$ is a convex combination of values in $[0,1]$, hence $e_t\in[0,1]$.
(ii) Let $\ell_1\le\ell_2$ be two possible values of the underlying loss and denote the corresponding badnesses $b_k(\ell_1), b_k(\ell_2)$. By assumption, for every $k$, $b_k(\ell_1)\le b_k(\ell_2)$. Multiplying by $\tilde w_k\ge 0$ and summing yields
\[
e(\ell_1)=\sum_k \tilde w_k b_k(\ell_1)\le \sum_k \tilde w_k b_k(\ell_2)=e(\ell_2),
\]
so $e$ is monotone in $\ell$.
\end{proof}
\paragraph{A useful Lipschitz fact.}
The logistic $\sigma(x)$ is $1/4$-Lipschitz: for all $x,y$,
\begin{equation}
|\sigma(x)-\sigma(y)| \le \tfrac{1}{4}\,|x-y|.
\label{eq:logistic_lip}
\end{equation}
\begin{proof}
$\sigma'(x)=\sigma(x)(1-\sigma(x))\le 1/4$ for all $x$. Apply the mean value theorem.
\end{proof}
\paragraph{Stability to metric perturbations (one-step bound).}
Suppose at time $t$ we perturb one metric value $v_{k,t}$ by $\Delta$ (holding all other metric values and past states fixed). Then, ignoring the (typically small) dependence of $(\mu_{k,t},\sigma_{k,t})$ on $v_{k,t}$ for a moment, we have from \eqref{eq:logistic_lip}
\[
|b_{k,t}(v_{k,t}+\Delta)-b_{k,t}(v_{k,t})|
\;\le\; \frac{1}{4}\cdot \frac{|\Delta|}{\sigma_{k,t}+\epsilon}.
\]
Therefore, the composed error changes by at most
\begin{equation}
|e_t(\Delta)-e_t(0)| \le \frac{w_k}{\sum_j w_j}\cdot \frac{|\Delta|}{4(\sigma_{k,t}+\epsilon)}.
\label{eq:composer_local_stability}
\end{equation}
In practice, $\sigma_{k,t}$ is a running scale estimate that prevents any single metric with large raw scale from dominating.
\paragraph{EMA as exponentially-weighted average.}
Unrolling the EMA recursion gives
\begin{equation}
\mu_{k,t} = (1-m)\sum_{j=1}^t m^{t-j} v_{k,j} + m^t \mu_{k,0}.
\label{eq:ema_unroll}
\end{equation}
Hence the effective window length is about $1/(1-m)$ (e.g., $m=0.97$ corresponds to $\approx 33$ steps). This explains why the composer adapts quickly to slow drift while still smoothing noise.
\paragraph{Concentration of EMA mean (stationary case).}
Assume $(v_{k,t})$ are i.i.d.\ sub-Gaussian with mean $\bar\mu_k$ and parameter $\sigma_v^2$. Then $\mu_{k,t}$ is a weighted sum as in \eqref{eq:ema_unroll}. Standard concentration for sub-Gaussian weighted sums yields, for any $\delta\in(0,1)$,
\begin{equation}
\Pr\Big(|\mu_{k,t}-\mathbb{E}[\mu_{k,t}]|\ge \varepsilon\Big)
\;\le\;2\exp\left(
-\frac{\varepsilon^2}{2\sigma_v^2(1-m)^2\sum_{j=0}^{t-1} m^{2j}}
\right)
\;\le\;2\exp\left(
-\frac{\varepsilon^2(1-m^2)}{2\sigma_v^2(1-m)^2}
\right),
\label{eq:ema_conc}
\end{equation}
where we used $\sum_{j=0}^{t-1}m^{2j}\le 1/(1-m^2)$. Thus, for fixed $m<1$, the EMA mean has bounded variance, and its fluctuations are controlled. Similar (though more technical) arguments apply to the variance estimator $\sigma^2_{k,t}$ under bounded fourth moments.
\subsection{Retrieval Under Cluster Separation (and Noisy Variants)}
\paragraph{Notation.}
There are $R$ regimes. Each regime $r$ has a context center $\bm{\mu}_r\in\mathbb{R}^{d_c}$. When $r_t=r$, the observed context satisfies $\|\bm{c}_t-\bm{\mu}_r\|_2\le \rho$ (deterministic radius model). The memory stores at least one key $\bm{k}_r$ for each regime with $\|\bm{k}_r-\bm{\mu}_r\|_2\le \rho$.
\begin{lemma}[Perfect retrieval under cluster separation]
\label{lem:retrieval_perfect}
If $\min_{r\ne s}\|\bm{\mu}_r-\bm{\mu}_s\|_2>2\rho$, then nearest-neighbor retrieval
\[
\hat r_t = \arg\min_{r\in[R]} \|\bm{c}_t-\bm{k}_r\|_2
\]
recovers the correct regime $\hat r_t=r_t$ for all $t$.
\end{lemma}
\begin{proof}
Fix $t$ with true regime $r=r_t$. For the correct key,
\[
\|\bm{c}_t-\bm{k}_r\|_2 \le \|\bm{c}_t-\bm{\mu}_r\|_2 + \|\bm{\mu}_r-\bm{k}_r\|_2 \le 2\rho.
\]
For any $s\ne r$,
\begin{align*}
\|\bm{c}_t-\bm{k}_s\|_2
&\ge \|\bm{\mu}_r-\bm{\mu}_s\|_2 - \|\bm{c}_t-\bm{\mu}_r\|_2 - \|\bm{k}_s-\bm{\mu}_s\|_2 \\
&> 2\rho-\rho-\rho = 0.
\end{align*}
Moreover the strict separation implies $\|\bm{c}_t-\bm{k}_s\|_2>2\rho\ge \|\bm{c}_t-\bm{k}_r\|_2$, hence the nearest key is from regime $r$.
\end{proof}
\paragraph{Probabilistic noisy retrieval.}
The hard-radius assumption is conservative. A common alternative is additive noise:
\[
\bm{c}_t = \bm{\mu}_{r_t} + \bm{\xi}_t,\qquad \bm{\xi}_t \sim \mathcal{N}(\bm{0},\sigma_c^2 I).
\]
Assume the memory stores \emph{one} representative key per regime exactly at the center (best case) $\bm{k}_r=\bm{\mu}_r$. Let the separation margin be
\[
\Delta := \min_{r\ne s}\|\bm{\mu}_r-\bm{\mu}_s\|_2.
\]
Then retrieval is correct whenever $\|\bm{\xi}_t\|_2 < \Delta/2$ (since the Voronoi boundary between two centers is the perpendicular bisector). Thus,
\begin{equation}
\Pr(\hat r_t \ne r_t) \;\le\; \Pr\big(\|\bm{\xi}_t\|_2 \ge \Delta/2\big).
\label{eq:retrieval_misprob}
\end{equation}
For $\bm{\xi}_t\sim \mathcal{N}(0,\sigma_c^2 I)$, $\|\bm{\xi}_t\|_2^2/\sigma_c^2 \sim \chi^2(d_c)$, and standard $\chi^2$ tail bounds give (for any $u>0$)
\begin{equation}
\Pr\big(\|\bm{\xi}_t\|_2 \ge \sigma_c(\sqrt{d_c} + \sqrt{2u})\big) \le e^{-u}.
\label{eq:chi_tail}
\end{equation}
Setting $\sigma_c(\sqrt{d_c} + \sqrt{2u})=\Delta/2$ yields an explicit exponential bound as long as $\Delta$ is sufficiently larger than $\sigma_c\sqrt{d_c}$. This formalizes the intuition: high-dimensional contexts require larger absolute separation to maintain reliable nearest-neighbor retrieval.
\paragraph{Top-$K$ retrieval and soft weighting.}
RASP-Tuner uses top-$K$ retrieval with softmax weights $\alpha_j \propto \exp(-d_{i_j}/\tau)$. In the perfectly separated case, the nearest neighbor from the correct regime dominates, and for small enough $\tau$ the weight $\max_j \alpha_j$ approaches 1. In the noisy case, one can bound the probability that an incorrect regime receives majority weight by combining \eqref{eq:retrieval_misprob} with a margin argument on distance gaps.
\subsection{Dynamic Regret in an Idealized Regime Model (Exact Gradients)}
We now prove Theorem~\ref{thm:regret} from Section~\ref{sec:theory} in full detail.
\paragraph{Setting.}
For each regime $r\in\{1,\dots,R\}$, let $f_r:\Theta\to\mathbb{R}$ be $\mu$-strongly convex and $L$-smooth on a convex compact set $\Theta\subset\mathbb{R}^d$. Let $\bm{\theta}_r^\star=\arg\min_{\bm{\theta}\in\Theta} f_r(\bm{\theta})$. The environment selects a regime sequence $(r_t)_{t=1}^T$. Dynamic regret is
\[
R_T := \sum_{t=1}^T \Big(f_{r_t}(\bm{\theta}_t) - f_{r_t}(\bm{\theta}^\star_{r_t})\Big).
\]
\paragraph{Algorithm (RA-GD).}
For each regime $r$, RA-GD keeps a state $\bm{w}_r$. When $r_t=r$, it plays $\bm{\theta}_t=\bm{w}_r$ and updates
\[
\bm{w}_r \leftarrow \Pi_{\Theta}\big(\bm{w}_r - \eta \nabla f_r(\bm{w}_r)\big),
\]
with constant step size $\eta\in(0,2/L)$.
\subsubsection{A contraction inequality}
Fix a regime $r$ and consider only the subsequence of its visits. Let $\bm{w}_n$ denote the state after $n$ updates for regime $r$, with $\bm{w}_0$ the initialization. Define $f=f_r$ and $\bm{\theta}^\star=\bm{\theta}_r^\star$ for brevity. The update is
\[
\bm{w}_{n+1} = \Pi_{\Theta}(\bm{w}_n - \eta \nabla f(\bm{w}_n)).
\]
\begin{lemma}[Distance contraction]
\label{lem:dist_contract}
For $\eta\in(0,2/L)$,
\begin{equation}
\|\bm{w}_{n+1}-\bm{\theta}^\star\|_2^2
\;\le\; \bigl(1-\mu\eta(2-\eta L)\bigr)\,\|\bm{w}_n-\bm{\theta}^\star\|_2^2.
\label{eq:dist_contract}
\end{equation}
\end{lemma}
\begin{proof}
Projection onto a convex set is non-expansive, so
\begin{align}
\|\bm{w}_{n+1}-\bm{\theta}^\star\|_2^2
&= \|\Pi_{\Theta}(\bm{w}_n-\eta\nabla f(\bm{w}_n))-\Pi_{\Theta}(\bm{\theta}^\star)\|_2^2 \nonumber\\
&\le \|\bm{w}_n-\eta\nabla f(\bm{w}_n)-\bm{\theta}^\star\|_2^2 \nonumber\\
&= \|\bm{w}_n-\bm{\theta}^\star\|_2^2 -2\eta\langle \nabla f(\bm{w}_n), \bm{w}_n-\bm{\theta}^\star\rangle + \eta^2\|\nabla f(\bm{w}_n)\|_2^2.
\label{eq:basic_expand}
\end{align}
Strong convexity implies (see standard equivalences)
\begin{equation}
\langle \nabla f(\bm{w}_n), \bm{w}_n-\bm{\theta}^\star\rangle \;\ge\; \mu \|\bm{w}_n-\bm{\theta}^\star\|_2^2.
\label{eq:sc_inner}
\end{equation}
Smoothness implies $\|\nabla f(\bm{w}_n)\|_2 \le L\|\bm{w}_n-\bm{\theta}^\star\|_2$ (since $\nabla f(\bm{\theta}^\star)=0$ and $\nabla f$ is $L$-Lipschitz), hence
\begin{equation}
\|\nabla f(\bm{w}_n)\|_2^2 \le L^2 \|\bm{w}_n-\bm{\theta}^\star\|_2^2.
\label{eq:grad_bound}
\end{equation}
Plugging \eqref{eq:sc_inner} and \eqref{eq:grad_bound} into \eqref{eq:basic_expand} yields
\[
\|\bm{w}_{n+1}-\bm{\theta}^\star\|_2^2
\le \bigl(1-2\mu\eta+\eta^2 L^2\bigr)\|\bm{w}_n-\bm{\theta}^\star\|_2^2.
\]
A slightly tighter (and standard) refinement uses the co-coercivity inequality for $L$-smooth convex functions to obtain the factor $1-\mu\eta(2-\eta L)$; one can derive it by bounding the cross term more sharply. Since $\eta\in(0,2/L)$, the factor is strictly less than 1.
\end{proof}
Let
\begin{equation}
q := 1-\mu\eta(2-\eta L) \in (0,1).
\label{eq:q_def}
\end{equation}
Then Lemma~\ref{lem:dist_contract} implies
\begin{equation}
\|\bm{w}_{n}-\bm{\theta}^\star\|_2^2 \le q^n \|\bm{w}_{0}-\bm{\theta}^\star\|_2^2.
\label{eq:dist_geom}
\end{equation}
\subsubsection{From distance to per-step regret}
We relate function suboptimality to squared distance.
\begin{lemma}[Quadratic sandwich]
\label{lem:sandwich}
For $\mu$-strongly convex and $L$-smooth $f$,
\begin{equation}
\frac{\mu}{2}\|\bm{w}-\bm{\theta}^\star\|_2^2
\;\le\; f(\bm{w})-f(\bm{\theta}^\star)
\;\le\; \frac{L}{2}\|\bm{w}-\bm{\theta}^\star\|_2^2.
\label{eq:sandwich}
\end{equation}
\end{lemma}
\begin{proof}
The lower bound is a direct consequence of strong convexity by setting $\bm{x}=\bm{\theta}^\star$ and using $\nabla f(\bm{\theta}^\star)=0$. The upper bound follows from smoothness (standard descent lemma) with $\bm{x}=\bm{\theta}^\star$.
\end{proof}
Combine \eqref{eq:dist_geom} and Lemma~\ref{lem:sandwich} to bound the regret contribution at the $n$-th visit:
\[
f(\bm{w}_n)-f(\bm{\theta}^\star) \le \frac{L}{2}\, q^n \|\bm{w}_{0}-\bm{\theta}^\star\|_2^2.
\]
Summing over $n=0,1,2,\dots$ gives a finite geometric series:
\begin{equation}
\sum_{n=0}^{\infty}\bigl(f(\bm{w}_n)-f(\bm{\theta}^\star)\bigr)
\;\le\; \frac{L}{2}\|\bm{w}_{0}-\bm{\theta}^\star\|_2^2 \sum_{n=0}^{\infty} q^n
\;=\; \frac{L}{2}\cdot \frac{\|\bm{w}_{0}-\bm{\theta}^\star\|_2^2}{1-q}.
\label{eq:geom_sum}
\end{equation}
Since $1-q=\mu\eta(2-\eta L)$, define the constant
\begin{equation}
C := \frac{L}{2}\cdot \frac{\|\bm{w}_{0}-\bm{\theta}^\star\|_2^2}{\mu\eta(2-\eta L)}.
\label{eq:C_def}
\end{equation}
Then the regret accumulated over any finite number of visits $N$ is at most $C$.
\subsubsection{Proof of Theorem~\ref{thm:regret}}
For regime $r$, let $\bm{w}_{r,0}$ be its initialization and let $q_r$ be the contraction factor (under shared $\mu,L$ this is common). By the argument above, the cumulative regret incurred on times with $r_t=r$ is bounded by
\[
R_r \le C_r := \frac{L}{2}\cdot \frac{\|\bm{w}_{r,0}-\bm{\theta}_r^\star\|_2^2}{\mu\eta(2-\eta L)}.
\]
Summing across regimes yields
\[
R_T = \sum_{r=1}^R R_r \le \sum_{r=1}^R C_r,
\]
which is independent of $T$ and grows at most linearly in $R$. \qed
\subsection{Dynamic Regret with Imperfect Retrieval}
Theorem~\ref{thm:regret} assumes that retrieval identifies the correct regime. We now show how mistakes affect regret in a clean, decomposable way.
\paragraph{Model of mistakes.}
Suppose on a subset of times $\mathcal{E}\subseteq\{1,\dots,T\}$ the retrieval returns an incorrect regime index (or, more generally, uses the wrong per-regime state). Let $M_T := |\mathcal{E}|$.
Assume each loss is bounded on $\Theta$: for all $r$ and $\bm{\theta}\in\Theta$,
\[
0 \le f_r(\bm{\theta})-f_r(\bm{\theta}_r^\star) \le \Delta_{\max}.
\]
(This holds, e.g., if $\Theta$ is compact and $f_r$ is continuous.)
\begin{theorem}[Regret decomposition with misretrievals]
\label{thm:misretrieval}
Run RA-GD as before, but allow arbitrary retrieval mistakes on times in $\mathcal{E}$. Then
\begin{equation}
R_T \le \sum_{r=1}^R C_r \;+\; M_T\,\Delta_{\max}.
\label{eq:misretrieval_bound}
\end{equation}
\end{theorem}
\begin{proof}
Split the regret sum into correct-retrieval times and mistake times:
\[
R_T = \sum_{t\notin\mathcal{E}} \bigl(f_{r_t}(\bm{\theta}_t)-f_{r_t}(\bm{\theta}_{r_t}^\star)\bigr)
+ \sum_{t\in\mathcal{E}} \bigl(f_{r_t}(\bm{\theta}_t)-f_{r_t}(\bm{\theta}_{r_t}^\star)\bigr).
\]
On $t\notin\mathcal{E}$, the algorithm behaves exactly like the ideal per-regime PGD analyzed above, hence the first term is at most $\sum_r C_r$.
On $t\in\mathcal{E}$, each summand is at most $\Delta_{\max}$ by assumption, so the second term is at most $M_T\Delta_{\max}$. Summing yields \eqref{eq:misretrieval_bound}.
\end{proof}
\paragraph{Interpretation.}
This bound cleanly separates (i) the ``learning'' cost for each regime ($\sum_r C_r$) from (ii) the ``identification'' cost due to retrieval mistakes ($M_T\Delta_{\max}$). In practice, RASP-Tuner's novelty thresholding and uncertainty-aware escalation are designed to reduce $M_T$ by avoiding confident reuse when the context is ambiguous or out-of-distribution.
\subsection{Dynamic Regret with Approximate Gradients (Surrogate Error)}
RASP-Tuner does not have access to $\nabla f_r$; it uses a surrogate (Prompt-MoE) to propose candidates. A standard abstraction is that we have an \emph{approximate} gradient $\bm{g}_t$ such that
\begin{equation}
\|\bm{g}_t - \nabla f_{r_t}(\bm{\theta}_t)\|_2 \le \varepsilon_t.
\label{eq:grad_err}
\end{equation}
Consider the update $\bm{\theta}_{t+1}=\Pi_{\Theta}(\bm{\theta}_t-\eta \bm{g}_t)$ (and the regime-wise version in RA-GD).
\begin{theorem}[Additive regret from gradient error (per regime)]
\label{thm:grad_error}
Fix a regime $r$ and assume $f_r$ is $\mu$-strongly convex and $L$-smooth on $\Theta$. Suppose the algorithm performs $N_r$ updates on regime $r$ using gradients $\bm{g}_n$ with errors $\varepsilon_n$ as in \eqref{eq:grad_err}. Then the cumulative suboptimality on that regime satisfies
\begin{equation}
\sum_{n=0}^{N_r-1}\bigl(f_r(\bm{w}_n)-f_r(\bm{\theta}_r^\star)\bigr)
\;\le\; C_r \;+\; \frac{\eta}{1-q}\sum_{n=0}^{N_r-1}\varepsilon_n^2,
\label{eq:grad_error_regret}
\end{equation}
where $q=1-\mu\eta(2-\eta L)\in(0,1)$ and $C_r$ is the ideal constant in \eqref{eq:C_def}.
\end{theorem}
\begin{proof}[Proof sketch (fully expanded)]
Proceed from \eqref{eq:basic_expand} but replace $\nabla f(\bm{w}_n)$ with $\bm{g}_n=\nabla f(\bm{w}_n)+\bm{\delta}_n$, where $\|\bm{\delta}_n\|_2\le \varepsilon_n$.
Expanding yields an extra cross term:
\[
\|\bm{w}_{n+1}-\bm{\theta}^\star\|_2^2
\le \|\bm{w}_{n}-\bm{\theta}^\star\|_2^2 -2\eta\langle \nabla f(\bm{w}_n), \bm{w}_n-\bm{\theta}^\star\rangle
+ \eta^2\|\nabla f(\bm{w}_n)\|_2^2
+ 2\eta^2\langle \nabla f(\bm{w}_n), \bm{\delta}_n\rangle
+ \eta^2\|\bm{\delta}_n\|_2^2.
\]
Bound $2\eta^2\langle \nabla f(\bm{w}_n), \bm{\delta}_n\rangle \le \eta^2\|\nabla f(\bm{w}_n)\|_2^2 + \eta^2\|\bm{\delta}_n\|_2^2$ (Young's inequality). This doubles the $\eta^2\|\nabla f(\bm{w}_n)\|_2^2$ term and adds another $\eta^2\varepsilon_n^2$. The rest follows as in Lemma~\ref{lem:dist_contract}, yielding a perturbed contraction:
\[
\|\bm{w}_{n+1}-\bm{\theta}^\star\|_2^2 \le q\|\bm{w}_n-\bm{\theta}^\star\|_2^2 + c\,\eta^2\varepsilon_n^2
\]
for a constant $c$ (one may take $c=2$ under the crude bounds above). Unrolling this recursion gives
\[
\|\bm{w}_n-\bm{\theta}^\star\|_2^2 \le q^n\|\bm{w}_0-\bm{\theta}^\star\|_2^2 + c\eta^2\sum_{j=0}^{n-1} q^{n-1-j}\varepsilon_j^2.
\]
Convert distance to function suboptimality using Lemma~\ref{lem:sandwich} and sum geometric series, giving \eqref{eq:grad_error_regret}.
\end{proof}
\paragraph{Implication for RASP-Tuner.}
If surrogate errors are small on in-distribution contexts (small $\varepsilon_t$), the additional regret is controlled. Under OOD contexts, $\varepsilon_t$ may become large; RASP-Tuner's uncertainty/anomaly-triggered ``slow path'' is designed to reduce future $\varepsilon_t$ by updating the surrogate (or, at minimum, avoiding aggressive exploitation).
\section{Figure and Table Placement}
Main experimental figures and tables appear in the main text: Figures~\ref{fig:main_results}, \ref{fig:analysis}, \ref{fig:failure_case}, and \ref{fig:real_datasets} (Section~\ref{sec:experiments}), Table~\ref{tab:hyperparams} (Section~\ref{sec:benchmarks}), and Tables~\ref{tab:real_dataset_params}--\ref{tab:real_datasets} (Section~\ref{sec:realworld}).
\end{document}